\documentclass[a4paper]{llncs}

\usepackage{alltt}
\usepackage{caption}
\usepackage{paralist}
\usepackage{nicefrac}
\usepackage{booktabs}
\usepackage{algorithm}
\usepackage{algorithmic}
\usepackage{xspace}
\usepackage{amsmath}
\usepackage{marvosym}
\usepackage{wasysym}
\usepackage{verbatim}
\usepackage{subfigure}
\usepackage{hyperref}
\usepackage{multirow}
\usepackage{cite}
\usepackage[capitalize]{cleveref}
\usepackage[per-mode=symbol,detect-all]{siunitx}
\usepackage{graphics}
\usepackage{amssymb}
\usepackage{wrapfig}
\usepackage{enumitem}
\usepackage[table]{xcolor}
\usepackage{pdfpages}
\usepackage{wrapfig}
\usepackage{colortbl}
\usepackage{graphicx}
\usepackage{caption}
\usepackage{subfigure}
\usepackage{tikz}
\usepackage{tikz-3dplot}
\usepackage{multirow}

\usepackage{tabularx,stackengine,collcell}
\let\endminwd\relax
\newcolumntype{L}[1]{>{\collectcell\xminwd l{#1}}l<{\endminwd\endcollectcell}}
\newcolumntype{C}[1]{>{\collectcell\xminwd c{#1}}c<{\endminwd\endcollectcell}}
\newcolumntype{R}[1]{>{\collectcell\xminwd r{#1}}r<{\endminwd\endcollectcell}}
\def\minwd#1#2#3\endminwd{\stackengine{0pt}{#3}{\rule{#2}{0pt}}{O}{#1}{F}{F}{L}}
\newcommand\xminwd[1]{\minwd#1}

\newcommand{\mysubsection}[1]{\medskip\noindent\textbf{#1}}

\definecolor{navy}{RGB}{0,0,128}

\renewcommand{\comment}[1]{}

\newcommand{\relu}{\text{ReLU}\xspace{}}

\newcommand{\sat}{\texttt{SAT}}
\newcommand{\unsat}{\texttt{UNSAT}}

\renewcommand{\phi}{\varphi}

\newcommand{\inc}{\texttt{inc}}
\newcommand{\dec}{\texttt{dec}}

\newcommand{\guard}{\mathcal{G}}

\newcommand{\realsat}{\texttt{RealSAT}}

\newcommand{\allvars}{\mathcal{X}}
\newcommand{\ub}{u}
\newcommand{\lb}{l}
\newcommand{\assignment}{\alpha{}}

\newcommand{\reluSet}{R}

\newcommand{\drule}[2]{
	\renewcommand{\arraystretch}{1.2}
	\(\begin{array}{c}
		#1 \\
		\hline
		#2
	\end{array}\)
}
\newcommand{\rulename}[1]{\ensuremath{\mathsf{#1}}\xspace}
\newcommand{\irulename}[2]{\ensuremath{\mathsf{#1}_{#2}}\xspace}

\newcommand{\failure}{\rulename{Failure}}

\newcommand{\realsuccess}{\rulename{RealSuccess}}

\newcommand{\reluSuccess}{\rulename{Success}}

\newcommand{\reluSplit}{\rulename{ReluSplit}}

\newcommand{\StepRefine}{\rulename{RefinementStep}}
\newcommand{\StepAbstract}{\rulename{AbstractionStep}}
\newcommand{\ApplyAbstraction}{\rulename{ApplyAbstraction}}
\newcommand{\Prune}[1]{\irulename{Prune}{#1}}

\newcommand{\arf}{$AR^4$}

\usepackage{tikz,ifthen,pgfplots}
\usetikzlibrary{arrows,trees,backgrounds,automata,shapes,decorations,plotmarks,fit,calc,positioning,shadows,chains}
\tikzstyle{every pin edge}=[<-,shorten <=1pt]
\tikzstyle{neuron}=[circle,fill=black!25,minimum size=17pt,inner sep=0pt]
\tikzstyle{input neuron}=[neuron, fill=green!50]
\tikzstyle{output neuron}=[neuron, fill=red!50]
\tikzstyle{hidden neuron}=[neuron, fill=blue!50]
\tikzstyle{merged neuron}=[neuron, fill=orange!50]
\tikzstyle{annot} = [text width=4em, text centered]
\tikzstyle{nnedge} = [-{stealth},shorten >=0.1cm, shorten <=0.05cm,line width=0.8pt,black]
\tikzstyle{startstop} = [rectangle, rounded corners, minimum width=1cm, minimum height=1cm,text centered, draw=black, fill=red!30]
\tikzstyle{io} = [trapezium, trapezium left angle=80, trapezium right angle=100, minimum width=0.8cm, minimum height=0.8cm, text centered, draw=black, fill=blue!30]
\tikzstyle{process} = [rectangle, minimum width=1cm, minimum height=1cm, text centered, text width=1.2cm, draw=black, fill=orange!30]
\tikzstyle{decision} = [diamond, minimum width=1cm, minimum height=1cm, text centered, draw=black, fill=green!30]
\tikzstyle{arrow} = [thick,->,>=stealth]

\usetikzlibrary{calc}

\renewcommand{\arraystretch}{1.2}

\begin{document}

\title{Neural Network Verification using \\Residual Reasoning}

\author{
  Yizhak Yisrael Elboher (\Letter)\orcidID{0000-0003-2309-3505}  \and
  Elazar Cohen \orcidID{0000-0002-2788-3574} \and
  Guy Katz \orcidID{0000-0001-5292-801X}
}

\institute{
  The Hebrew University of Jerusalem, Jerusalem, Israel \\
  \email\{yizhak.elboher, elazar.cohen1, g.katz\}@mail.huji.ac.il 
}

\maketitle	
\begin{abstract}
  With the increasing integration of neural networks as components in
  mission-critical systems, there is an increasing need to ensure that
  they satisfy various safety and liveness requirements. In recent
  years, numerous sound and complete verification methods have been
  proposed towards that end, but these typically suffer from severe
  scalability limitations. Recent work has proposed enhancing such
  verification techniques with abstraction-refinement capabilities,
  which have been shown to boost scalability: instead of verifying a
  large and complex network, the verifier constructs and then verifies
  a much smaller network, whose correctness implies the correctness of
  the original network. A shortcoming of such a scheme is that if
  verifying the smaller network fails, the verifier needs to perform a
  refinement step that increases the size of the network being
  verified, and then start verifying the new network from scratch ---
  effectively ``wasting'' its earlier work on verifying the smaller
  network.  In this paper, we present an enhancement to
  abstraction-based verification of neural networks, by using
  \emph{residual reasoning}: the process of utilizing information
  acquired when verifying an abstract network, in order to expedite
  the verification of a refined network. In essence, the method allows
  the verifier to store information about parts of the search space in
  which the refined network is guaranteed to behave correctly, and
  allows it to focus on areas where bugs might be discovered. We
  implemented our approach as an extension to the Marabou verifier,
  and obtained promising results.
\end{abstract}

\section{Introduction}

In recent years, the use of deep neural networks
(DNNs)~\cite{GoodBengCour16} in critical components of diverse systems
has been gaining momentum. A few notable examples include the fields
of speech recognition~\cite{devlin-etal-2019-bert}, image
recognition~\cite{heZaReSu2015deep}, autonomous
driving~\cite{BoDeDwFiFlGoJaMoMuZhZhZhZi16}, and many others. The
reason for this unprecedented success is the ability of DNNs to
generalize from a small set of training data, and then correctly
handle previously unseen inputs.

Still, despite their success, neural networks suffer from various
reliability issues. First, they are completely dependent on the
training process, which may include data that is anecdotal, partial,
noisy, or biased~\cite{KiKiKiKiKi2019, SoKiPaShLe2020}; further, the
training process has inherent over-fitting
limitations~\cite{JPhCS1168b2022Y2019}; and finally, trained networks
suffer from susceptibility to adversarial
attacks, as well as from obscurity and lack of
explainability~\cite{ANGELOV2020185}.  Unless addressed, these
concerns, and others, are likely to limit the applicability of DNNs in
the coming years.

A promising approach for improving the reliability of DNN models is to
apply \emph{formal verification} techniques: automated and rigorous
techniques that can ensure that a DNN model adheres to a given
specification, in all possible corner cases~\cite{Katz2017Reluplex,
  HuKwWaWu17, GeMiDrTsChVe18, WaPeWhYaJa18}.
While sound and complete formal verification methods can certify that
DNNs are reliable, these methods can typically only tackle small or
medium-sized DNNs; and despite significant strides in recent years,
scalability remains a major issue~\cite{BaLiJo21}.
  
In order to improve the scalability of DNN verification, recent
studies have demonstrated the great potential of enhancing it with
abstraction-refinement techniques~\cite{ClGrJhLuVe00CEGAR, ElGoKa20,
  AsHaKrMu20, PrAf20}. The idea is to use a black-box DNN verifier,
and feed it a series of \emph{abstract networks} --- i.e., DNNs that
are significantly smaller than the original network being
verified. Because the complexity of DNN verification is exponential in
the size of the DNN in question~\cite{Katz2017Reluplex}, these queries
can be solved relatively quickly; and the abstract networks are
constructed so that their correctness implies the correctness of the
original, larger network. The downside of abstraction is that
sometimes, verifying the smaller network returns an inconclusive
result --- in which case, the abstract network is \emph{refined} and
made slightly larger, and the process is repeated. Is it well known
that the heuristics used for performing the abstraction and refinement
steps can have a significant impact on performance~\cite{ClGrJhLuVe00CEGAR,
  ElGoKa20}, and that poor heuristics can cause the
abstraction-refinement sequence of queries to take longer to dispatch
than the original query.

In this paper, we propose an extension that can improve the
performance of an abstraction-refinement verification scheme.
The idea is to use \textit{residual reasoning}~\cite{AzCoPa20}: an approach for
re-using information obtained in an early verification query, in order
to expedite a subsequent query. Presently, a verifier might verify an
abstract network $N_1$, obtain an inconclusive answer, and then 
 verify a refined network, $N_2$; and it will verify $N_2$
from scratch, as if it had never verified $N_1$. Using residual
reasoning, we seek to leverage the similarities between $N_1$ and
$N_2$ in order to identify large portions of the verification search
space that need not be explored, because we are guaranteed a-priori
that they contain no violations of the property being checked. 

More specifically, modern verifiers can be regarded as traversing a
large search tree. Each branching in the tree is caused by an
\emph{activation function} within the neural network, which can take
on multiple linear phases; and each branch corresponds to one of these
phases. We show that when a verifier traverses a branch of the search
tree and determines that no property violations occur therein, that
information can be used to deduce that no violation can exist in some
of the branches of the search tree traversed when verifying a refined
network. The advantages of this approach are clear: by curtailing the
search space, the verification process can be expedited
significantly. The disadvantage is that, unlike in other
abstraction-refinement based techniques, the verifier needs to be
instrumented, and cannot be used as a black box.

Our contributions in this paper are as follows:
\begin{inparaenum}[(i)]
\item we formally define our residual reasoning scheme, in a general
  way that preserves the soundness and completeness of the underlying
  verifier;
\item we specify how our approach can be used to extend the
  state-of-the-art Marabou DNN verification
  engine~\cite{Katz2019Marabou}; and
\item we implement  our approach, and  evaluate it  on the ACAS
  Xu set of benchmarks~\cite{JuLoBrOwKo2016}.
\end{inparaenum}
We regard this work as a step towards tapping into the great potential
of abstraction-refinement methods in the context of DNN verification.

The rest of the paper is organized as follows. In
Section~\ref{sec:preliminaries} we recap the necessary background on
DNNs and their verification. Next, in
Section~\ref{sec:residual_reasoning} we describe our general method
for residual reasoning; followed by a discussion of how our technique
can enhance a specific abstraction-refinement method, in
Section~\ref{sec:rr_instantiation}.  Sections~\ref{sec:ar4} is then
dedicated to explaining how our method can be applied using the
Marabou DNN verifier as a backend, followed by our evaluation of the
approach in Section~\ref{sec:evaluation}.  Related work is covered in
Section~\ref{sec:related-work}, and we conclude in
Section~\ref{sec:conclusion}.

\section{Background}\label{sec:preliminaries}
	
\mysubsection{Deep Neural Networks (DNNs).}
A neural network~\cite{GoodBengCour16}
$N:\mathbb{R}^n\rightarrow\mathbb{R}^m$ is a directed graph, organized
into an input layer, multiple hidden layers, and an output layer. Each
layer is a set of nodes (neurons), which can take on real values. When
an input vector is passed into the input layer, it can be used to
iteratively compute the values of neurons in the following layers, all
through to neurons in the output layer --- which constitute the
network's output.  We use $ L_i $ to denote the $ i $'th layer of the
DNN, and $ v_{i,j} $ to denote the $ j $'th node in $ L_i $.

\begin{wrapfigure}[14]{r}{5.0cm}
  \vspace{-1.2cm}
  \begin{center}
    \scalebox{0.7} {
    	\def\layersep{2cm}
    	\begin{tikzpicture}[shorten >=1pt,->,draw=black!50, node distance=\layersep,font=\footnotesize]
    		
    		\path[yshift=0cm] node[input neuron, label={[xshift=-0.4cm]\textcolor{purple}{0}}] (I-1) at (0,0) {$x_1$};
    		\path[yshift=0cm] node[input neuron, label={[xshift=-0.4cm]\textcolor{purple}{1}}] (I-2) at (0,-1) {$x_2$};
    		
    		\node[hidden neuron, label={[xshift=-0.4cm]\textcolor{purple}{0}}] (H-1)
    		at (\layersep,1 cm) {$v_{1,1}$};
    		\node[hidden neuron, label={[xshift=-0.4cm]\textcolor{purple}{0}}] (H-2)
    		at (\layersep,0 cm) {$v_{1,2}$};
    		\node[hidden neuron, label={[xshift=-0.4cm]\textcolor{purple}{2}}] (H-3)
    		at (\layersep,-1 cm) {$v_{1,3}$};
    		\node[hidden neuron, label={[xshift=-0.4cm]\textcolor{purple}{1}}] (H-4)
    		at (\layersep,-2 cm) {$v_{1,4}$};
    		
    		\node[hidden neuron, label={[xshift=-0.2cm]\textcolor{purple}{-2}}] (H-5)
    		at (2*\layersep,1.5 cm) {$v_{2,1}$};
    		\node[hidden neuron, label={[xshift=-0.2cm]\textcolor{purple}{-4}}] (H-6)
    		at (2*\layersep,0.5 cm) {$v_{2,2}$};
    		\node[hidden neuron, label={[xshift=-0.2cm]\textcolor{purple}{0}}] (H-7)
    		at (2*\layersep,-0.5 cm) {$v_{2,3}$};
    		\node[hidden neuron, label={[xshift=-0.2cm]\textcolor{purple}{8}}] (H-8)
    		at (2*\layersep,-1.5 cm) {$v_{2,4}$};
    		\node[hidden neuron, label={[xshift=-0.2cm]\textcolor{purple}{1}}] (H-9)
    		at (2*\layersep,-2.5 cm) {$v_{2,5}$};
    		
    		\node[output neuron, label={[xshift=-0.2cm]\textcolor{purple}{9}}] at (3*\layersep, -0.5) (O-1) {$y$};
    		
    		\draw[nnedge] (I-1) -- node[above] {$1$} (H-1);
    		\draw[nnedge] (I-1) -- node[above] {$2$} (H-2);
    		
    		\draw[nnedge] (I-2) -- node[above] {$2$} (H-3);
    		\draw[nnedge] (I-2) -- node[above] {$1$} (H-4);
    		
    		\draw[nnedge] (H-1) -- node[above] {$3$} (H-5);
    		\draw[nnedge] (H-1) -- node[above] {$2$} (H-6);
    		\draw[nnedge] (H-2) -- node[above, pos=0.2] {$1$} (H-7);
    		\draw[nnedge] (H-3) -- node[left, pos=0.1] {$-1$} (H-5);
    		\draw[nnedge] (H-3) -- node[below, pos=0.3] {$-2$} (H-6);
    		\draw[nnedge] (H-4) -- node[above] {$8$} (H-8);
    		\draw[nnedge] (H-4) -- node[below] {$1$} (H-9);
    		
    		\draw[nnedge] (H-5) -- node[above] {$1$} (O-1);
    		\draw[nnedge] (H-6) -- node[above] {$1$} (O-1);
    		\draw[nnedge] (H-7) -- node[above, pos=0.3] {$-4$} (O-1);
    		\draw[nnedge] (H-8) -- node[above] {$1$} (O-1);
    		\draw[nnedge] (H-9) -- node[above] {$1$} (O-1);						
    	\end{tikzpicture}
    }

  \end{center}
    \caption{A DNN with an input layer  (green), two hidden layers
      (blue), and an output layer (red). 
    }
  \label{fig:simple-neural-network}
\end{wrapfigure}
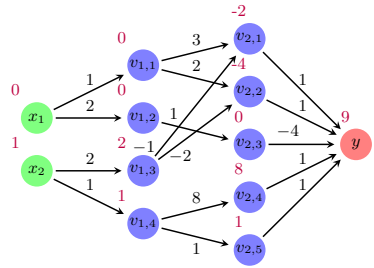
Typically, each neuron in the DNN is evaluated by first computing a
weighted sum of the values assigned to neurons in the preceding layer,
and then applying some activation function to the result. For
simplicity, we restrict our attention to the popular \relu{}
activation function~\cite{GoodBengCour16}, which is a piecewise-linear
function defined as $ \relu(x)=max(x,0) $.  When $x>0$, we say that
the \relu{} is active; and otherwise, we say that it is inactive. A
simple example appears in Fig.~\ref{fig:simple-neural-network}, and
shows a DNN evaluated on input $\langle 0,1\rangle$. The value above
each neuron is the weighted sum that it computes, prior to the
application of the \relu{} activation function. The network's output
is 9.

\mysubsection{Neural Network Verification.}  Neural network
verification~\cite{Liu2021AlgsVNN} deals with checking whether an
input-output relation in a neural network holds. A verification query
is a couple $ \langle N, \phi\rangle $, where $ N $ is a neural
network and $ \phi $ is a property of the form:
$ \vec{x}\in D_I \wedge \vec{y}\in D_O $, meaning that the
input $ \vec{x} $ is in some input domain $ D_I $ and the output
$ \vec{y} $ is in some output domain $ D_O $. Typically, $\phi$
represents \emph{undesirable} behavior; and so the verification
problem is to find an input $\vec{x}$ and its matching output
$\vec{y}$ that satisfy $ \phi $, and so constitute a counter-example
(the \sat{} case), or to prove that no such $\vec{x}$ exists (the
\unsat{} case). Without loss of generality, we assume that
verification queries only consists of a network $N$ with a single
output neuron $y$, and of a property $\varphi$ of the form 
$ \vec{x}\in D_I \wedge y>c$; other queries can be reduced to this
setting in a straightforward way~\cite{ElGoKa20}.

As a simple example, consider the DNN in
Fig~\ref{fig:simple-neural-network} and the property
$ \phi : x_1, x_2 \in [0,1] \wedge y > 14 $.  Checking whether input
$ x_1=0, x_2=1 $ satisfies this property, we get that it does not,
since $ y=9\le14 $. A sound verifier, therefore, would not return
$ \langle 0,1\rangle $ as a satisfying assignment for this
query.

\subsubsection{Linear Programming and Case Splitting.}
A key technique for DNN verification, which is nowadays used by many
leading verification tools, is called \emph{case
  splitting}~\cite{Katz2019Marabou, TjXiTe17, WaZhXuLiJaHsKo21}. A DNN
verification problem can be regarded as a satisfiability problem,
where linear constraints and \relu{} constraints must be satisfied
simultaneously; and while linear constraints are easy to
solve~\cite{Dantzig1963}, the \relu{}s render the problem
NP-Complete~\cite{Katz2017Reluplex}. In case splitting, the verifier
sometimes transforms a \relu{} constraint into an equivalent
disjunction of linear constraints:
\[
  (y=\relu{}(x)) \equiv \left((x\leq 0 \wedge y = 0)\vee (x\geq 0 \wedge y=x)\right)
\]
and then each time \emph{guesses} which of the two disjuncts holds,
and attempts to satisfy the resulting constraints. This approach gives
rise to a search tree, where internal nodes correspond to \relu{}
constraints, and their outgoing edges to the two linear constraints
each \relu{} can take. Each leaf of this tree is a problem that can be
solved directly, e.g., because all \relu{}s have been split
upon. These problems are often dispatched using linear programming
engines.

Case splitting might produce an exponential number of
sub-problems, and so solvers apply a myriad of heuristics to
avoid them or prioritize between them. Solvers also use deduction to
rule out a-priori case splits that cannot lead to a satisfying
assignment. Such techniques are beyond our scope.

\mysubsection{Abstraction-Refinement (AR).}
Abstraction-refinement is a common mechanism for improving the
performance of verification tools in various
domains~\cite{ClGrJhLuVe00CEGAR}, including in DNN
verification~\cite{ElGoKa20, AsHaKrMu20, PrAf20}.  A
sketch of the basic scheme of AR is illustrated in
Fig.~\ref{fig:AR-mechanism} in Appendix~\ref{appendix:CEGAR}. The process begins with a DNN $N$ and a
property $\phi$ to verify, and then \emph{abstracts} $N$ into a
different, smaller network $ N' $. A key property  is
that $N'$ \emph{over-approximates} $N$: if
$\langle N', \varphi\rangle$ is \unsat{}, then
$\langle N, \varphi\rangle$ is also \unsat{}. Thus, it is usually
preferable to verify the smaller $N'$ instead of $N$.

If a verifier determines that $\langle N', \varphi\rangle$ is \sat{},
it returns a counter-example $\vec{x_0}$. That counter-example is then
checked to determine whether it also constitutes a counterexample for
$\langle N, \varphi\rangle$. If so, the original query is \sat{}, and
we are done; but otherwise, $\vec{x_0}$ is a \emph{spurious}
counter-example, indicating that $N'$ is inadequate for determining the
satisfiability of the original query. We then apply \emph{refinement}:
we use $N'$, and usually also $\vec{x_0}$, to create a new network
$N''$, which is larger than $N'$ but is still an over-approximation of
$N$. The process is then repeated using $N''$. Usually, the process is
guaranteed to converge: either we are able to determine the
satisfiability of the original query using one of the abstract
networks, or we end up refining $N'$ all the way back to $N$, and
solve the original query, which, by definition, cannot return a
spurious result.

In this paper we focus on a particular abstraction-refinement
mechanism for DNN verification~\cite{ElGoKa20}. There,
abstraction and refinement are performed by merging or
splitting (respectively) neurons in the network, and
aggregating the weights of their incoming and outgoing edges.  This
merging and splitting is carried out in a specific way, which
guarantees that if $N$ is abstracted into $N'$, then for all input
$\vec{x}$ it holds that $N'(\vec{x})\geq N(\vec{x})$; and thus, if
$N'(\vec{x})\geq c$ is \unsat{}, then $N(\vec{x})\geq c$ is also
\unsat{}, as is required of an over-approximation.

An illustrative example appears in Fig.~\ref{fig:running-example}.
On the left, we have  the network from
Fig.~\ref{fig:simple-neural-network}, denoted $N$. The middle network,
denoted $N'$, is obtained by merging together neurons $v_{2,1}$ and
$v_{2,2}$ into the single neuron $v_{2,1+2}$; and by merging neurons
$v_{2,4}$ and $v_{2,5}$ into the single neuron $v_{2,4+5}$. The
weights on the outgoing edges of these neurons are the sums of the
outgoing edges of their original neurons; and the weights of the
incoming edges are either the $\min$ or $\max$ or the original
weights, depending on various criteria~\cite{ElGoKa20}. It can be
proven~\cite{ElGoKa20} that $N'$ over-approximates $N$; for example,
$N(\langle 3,1\rangle)=-6 < N'(\langle 3,1\rangle)= 6$.
Finally, the network on the right, denoted $N''$, is obtained from $N$
by splitting a previously merged neuron. $N''$ is larger than $N'$,
but it is still an over-approximation of the original $N$: for
example,
$N''(\langle 3,1\rangle)=1 > N(\langle 3,1\rangle) = -6$.

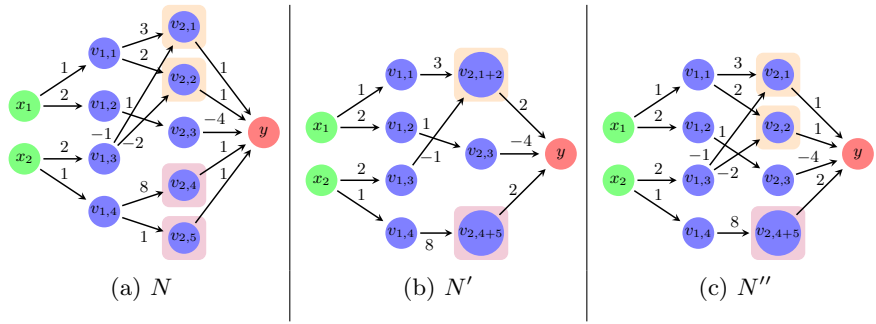
\begin{figure}
	\begin{tabular}{c|c|c}
		 \scalebox{0.7} {
		 	\def\layersep{1.5cm}
		 	\begin{tikzpicture}[shorten >=1pt,->,draw=black!50, node distance=\layersep,font=\footnotesize]
		 		
		 		\path[yshift=0cm] node[input neuron] (I-1) at (0,0) {$x_1$};
		 		\path[yshift=0cm] node[input neuron] (I-2) at (0,-1) {$x_2$};
		 		
		 		\node[hidden neuron] (H-1)
		 		at (\layersep,1 cm) {$v_{1,1}$};
		 		\node[hidden neuron] (H-2)
		 		at (\layersep,0 cm) {$v_{1,2}$};
		 		\node[hidden neuron] (H-3)
		 		at (\layersep,-1 cm) {$v_{1,3}$};
		 		\node[hidden neuron] (H-4)
		 		at (\layersep,-2 cm) {$v_{1,4}$};
		 		
		 		\node[hidden neuron] (H-5)
		 		at (2*\layersep,1.5 cm) {$v_{2,1}$};
		 		\node[hidden neuron] (H-6)
		 		at (2*\layersep,0.5 cm) {$v_{2,2}$};
		 		\node[hidden neuron] (H-7)
		 		at (2*\layersep,-0.5 cm) {$v_{2,3}$};
		 		\node[hidden neuron] (H-8)
		 		at (2*\layersep,-1.5 cm) {$v_{2,4}$};
		 		\node[hidden neuron] (H-9)
		 		at (2*\layersep,-2.5 cm) {$v_{2,5}$};
		 		
		 		\node[output neuron] at (3*\layersep, -0.5) (O-1) {$y$};
		 		
		 		\draw[nnedge] (I-1) -- node[above] {$1$} (H-1);
		 		\draw[nnedge] (I-1) -- node[above] {$2$} (H-2);
		 		
		 		\draw[nnedge] (I-2) -- node[above] {$2$} (H-3);
		 		\draw[nnedge] (I-2) -- node[above] {$1$} (H-4);
		 		
		 		\draw[nnedge] (H-1) -- node[above] {$3$} (H-5);
		 		\draw[nnedge] (H-1) -- node[above] {$2$} (H-6);
		 		\draw[nnedge] (H-2) -- node[above, pos=0.2] {$1$} (H-7);
		 		\draw[nnedge] (H-3) -- node[left, pos=0.1] {$-1$} (H-5);
		 		\draw[nnedge] (H-3) -- node[below, pos=0.3] {$-2$} (H-6);
		 		\draw[nnedge] (H-4) -- node[above] {$8$} (H-8);
		 		\draw[nnedge] (H-4) -- node[below] {$1$} (H-9);
		 		
		 		\draw[nnedge] (H-5) -- node[above] {$1$} (O-1);
		 		\draw[nnedge] (H-6) -- node[above] {$1$} (O-1);
		 		\draw[nnedge] (H-7) -- node[above, pos=0.3] {$-4$} (O-1);
		 		\draw[nnedge] (H-8) -- node[above] {$1$} (O-1);
		 		\draw[nnedge] (H-9) -- node[above] {$1$} (O-1);						
		 		
		 		\begin{pgfonlayer}{background}
		 			\newcommand*{\nodePadding}{0.2cm}
		 			\tikzstyle{background_rectangle}=[rounded corners, fill = navy!10]
		 			
		 			\draw[fill = orange!20, draw=none, rounded corners]
		 			($(H-5.south west) + (-\nodePadding, -\nodePadding ) $)
		 			rectangle
		 			($(H-5.north east) + ( \nodePadding, \nodePadding) $);
		 			
		 			\draw[fill = orange!20, draw=none, rounded corners]
		 			($(H-6.south west) + (-\nodePadding, -\nodePadding ) $)
		 			rectangle
		 			($(H-6.north east) + ( \nodePadding, \nodePadding) $);
		 			
		 			\draw[fill = purple!20, draw=none, rounded corners]
		 			($(H-8.south west) + (-\nodePadding, -\nodePadding ) $)
		 			rectangle
		 			($(H-8.north east) + ( \nodePadding, \nodePadding) $);
		 			
		 			\draw[fill = purple!20, draw=none, rounded corners]
		 			($(H-9.south west) + (-\nodePadding, -\nodePadding ) $)
		 			rectangle
		 			($(H-9.north east) + ( \nodePadding, \nodePadding) $);
		 			
		 		\end{pgfonlayer}
		 		
		 	\end{tikzpicture}
		 }
		 &  
		 \scalebox{0.7} {
		 	\def\layersep{1.5cm}
		 	\begin{tikzpicture}[shorten >=1pt,->,draw=black!50, node distance=\layersep,font=\footnotesize]
		 		
		 		\path[yshift=0cm] node[input neuron] (I-1) at (0,0) {$x_1$};
		 		\path[yshift=0cm] node[input neuron] (I-2) at (0,-1) {$x_2$};
		 		
		 		\node[hidden neuron] (H-1)
		 		at (\layersep,1 cm) {$v_{1,1}$};
		 		\node[hidden neuron] (H-2)
		 		at (\layersep,0 cm) {$v_{1,2}$};
		 		\node[hidden neuron] (H-3)
		 		at (\layersep,-1 cm) {$v_{1,3}$};
		 		\node[hidden neuron] (H-4)
		 		at (\layersep,-2 cm) {$v_{1,4}$};
		 		
		 		\node[hidden neuron] (H-5)
		 		at (2*\layersep,1 cm) {$v_{2,1+2}$};
		 		\node[hidden neuron] (H-6)
		 		at (2*\layersep,-0.5 cm) {$v_{2,3}$};
		 		\node[hidden neuron] (H-7)
		 		at (2*\layersep,-2 cm) {$v_{2,4+5}$};
		 		
		 		\node[output neuron] at (3*\layersep, -0.5) (O-1) {$y$};
		 		
		 		\draw[nnedge] (I-1) -- node[above] {$1$} (H-1);
		 		\draw[nnedge] (I-1) -- node[above] {$2$} (H-2);
		 		
		 		\draw[nnedge] (I-2) -- node[above] {$2$} (H-3);
		 		\draw[nnedge] (I-2) -- node[above] {$1$} (H-4);
		 		
		 		\draw[nnedge] (H-1) -- node[above] {$3$} (H-5);
		 		\draw[nnedge] (H-2) -- node[above, pos=0.2] {$1$} (H-6);
		 		\draw[nnedge] (H-3) -- node[below,
                                pos=0.3] {$\ -1$} (H-5);
		 		\draw[nnedge] (H-4) -- node[below, pos=0.3] {$8$} (H-7);
		 		
		 		\draw[nnedge] (H-5) -- node[above] {$2$} (O-1);
		 		\draw[nnedge] (H-6) -- node[above] {$-4$} (O-1);
		 		\draw[nnedge] (H-7) -- node[above, pos=0.3] {$2$} (O-1);
		 		
		 		\begin{pgfonlayer}{background}
		 			\newcommand*{\nodePadding}{0.2cm}
		 			\tikzstyle{background_rectangle}=[rounded corners, fill = navy!10]
		 			
		 			\draw[fill = orange!20, draw=none, rounded corners]
		 			($(H-5.south west) + (-\nodePadding, -\nodePadding ) $)
		 			rectangle
		 			($(H-5.north east) + ( \nodePadding, \nodePadding) $);
		 			
		 			\draw[fill = purple!20, draw=none, rounded corners]
		 			($(H-7.south west) + (-\nodePadding, -\nodePadding ) $)
		 			rectangle
		 			($(H-7.north east) + ( \nodePadding, \nodePadding) $);
		 			
		 		\end{pgfonlayer}
		 		
		 	\end{tikzpicture}
		 }
		 & 
		 \scalebox{0.7} {
		 	\def\layersep{1.5cm}
		 	\begin{tikzpicture}[shorten >=1pt,->,draw=black!50, node distance=\layersep,font=\footnotesize]
		 		
		 		\path[yshift=0cm] node[input neuron] (I-1) at (0,0) {$x_1$};
		 		\path[yshift=0cm] node[input neuron] (I-2) at (0,-1) {$x_2$};
		 		
		 		\node[hidden neuron] (H-1)
		 		at (\layersep,1 cm) {$v_{1,1}$};
		 		\node[hidden neuron] (H-2)
		 		at (\layersep,0 cm) {$v_{1,2}$};
		 		\node[hidden neuron] (H-3)
		 		at (\layersep,-1 cm) {$v_{1,3}$};
		 		\node[hidden neuron] (H-4)
		 		at (\layersep,-2 cm) {$v_{1,4}$};
		 		
		 		\node[hidden neuron] (H-5)
		 		at (2*\layersep,1 cm) {$v_{2,1}$};
		 		\node[hidden neuron] (H-6)
		 		at (2*\layersep,0 cm) {$v_{2,2}$};
		 		\node[hidden neuron] (H-7)
		 		at (2*\layersep,-1 cm) {$v_{2,3}$};
		 		\node[hidden neuron] (H-8)
		 		at (2*\layersep,-2 cm) {$v_{2,4+5}$};
		 		
		 		\node[output neuron] at (3*\layersep, -0.5) (O-1) {$y$};
		 		
		 		\draw[nnedge] (I-1) -- node[above] {$1$} (H-1);
		 		\draw[nnedge] (I-1) -- node[above] {$2$} (H-2);
		 		
		 		\draw[nnedge] (I-2) -- node[above] {$2$} (H-3);
		 		\draw[nnedge] (I-2) -- node[above] {$1$} (H-4);
		 		
		 		\draw[nnedge] (H-1) -- node[above] {$3$} (H-5);
		 		\draw[nnedge] (H-1) -- node[above] {$2$} (H-6);
		 		\draw[nnedge] (H-2) -- node[above, pos=0.2] {$1$} (H-7);
		 		\draw[nnedge] (H-3) -- node[left, pos=0.15] {$-1$} (H-5);
		 		\draw[nnedge] (H-3) -- node[below, pos=0.3] {$-2$} (H-6);
		 		\draw[nnedge] (H-4) -- node[above] {$8$} (H-8);
		 		
		 		\draw[nnedge] (H-5) -- node[above] {$1$} (O-1);
		 		\draw[nnedge] (H-6) -- node[above] {$1$} (O-1);
		 		\draw[nnedge] (H-7) -- node[above, pos=0.3] {$-4$} (O-1);
		 		\draw[nnedge] (H-8) -- node[above] {$2$} (O-1);
		 		
		 		\begin{pgfonlayer}{background}
		 			\newcommand*{\nodePadding}{0.2cm}
		 			\tikzstyle{background_rectangle}=[rounded corners, fill = navy!10]
		 			
		 			\draw[fill = orange!20, draw=none, rounded corners]
		 			($(H-5.south west) + (-\nodePadding, -\nodePadding ) $)
		 			rectangle
		 			($(H-5.north east) + ( \nodePadding, \nodePadding) $);
		 			
		 			\draw[fill = orange!20, draw=none, rounded corners]
		 			($(H-6.south west) + (-\nodePadding, -\nodePadding ) $)
		 			rectangle
		 			($(H-6.north east) + ( \nodePadding, \nodePadding) $);
		 			
		 			\draw[fill = purple!20, draw=none, rounded corners]
		 			($(H-8.south west) + (-\nodePadding, -\nodePadding ) $)
		 			rectangle
		 			($(H-8.north east) + ( \nodePadding, \nodePadding) $);
		 			
		 		\end{pgfonlayer}
		 		
		 	\end{tikzpicture}
		 }
		 \\
		(a) $ N $ & (b) $ N' $ & (c) $ N'' $ \\[6pt]
	\end{tabular}
	\caption{Neural network abstraction and refinement through the
		merging and splitting of neurons~\cite{ElGoKa20}.}
	\label{fig:running-example}
\end{figure}

\section{Residual Reasoning (RR)}\label{sec:residual_reasoning}
Consider again our running example, and observe that property
$ \phi $ is satisfiable for the most abstract network: for
$\vec{x_0}=\langle 0,1\rangle$ we have $N'(\vec{x_0})=16$. However,
this $\vec{x_0}$ is a spurious counterexample, as
$N(\vec{x_0})=9$. Consequently, refinement is performed, and the
verifier sets out to verify $\langle N'', \varphi\rangle$; and this
query is solved from scratch. However, notice that the verification
queries of $ \phi $ in $ N', N'' $ are very similar: 
the networks are almost identical, and the property is
the same. The idea is thus to re-use some of the information already
discovered when $\langle N', \varphi\rangle$ was solved in order to
expedite the solving of $\langle N'', \varphi\rangle$. Intuitively, an
abstract network allows the verifier to explore the search space very
coarsely, whereas a refined network allows the verifier to explore
that space in greater detail. Thus, areas of that space that were
determined safe for the abstract network need not be re-explored in
the refined network.

In order to enable knowledge retention between subsequent calls to the
verifier, we propose to introduce a \emph{context} variable, $\Gamma$,
that is passed to  the verifier along with each verification query.
$\Gamma$ is used in two ways:
\begin{inparaenum}[(i)]
  \item the verifier can store into $\Gamma$ information that may be
    useful if a refined version of the current network is later
    verified; and
  \item the verifier may use information already in $\Gamma$ to
    curtail the search space of the query currently being solved.
  \end{inparaenum}
A scheme of the proposed mechanism appears in
Fig.~\ref{fig:AR-RR-mechanisms} in Appendix~\ref{appendix:CEGAR}.
Of course, 
$\Gamma$ must be designed carefully in order to maintain soundness.

\subsubsection{Avoiding Case-Splits with $\Gamma$.}
In order to expedite subsequent verification queries, we propose to
store in $\Gamma$ information that will allow the verifier to
\emph{avoid case splits}. Because case splits are the
most significant bottleneck in DNN
verification~\cite{Katz2017Reluplex, TjXiTe17}, using
$\Gamma$ to reduce their number seems
like a natural strategy.

Let $N'$ be an abstract network, and $N''$ its refinement; and observe
the queries 
 $\langle N', \varphi\rangle$ and $\langle N'', \varphi\rangle$. Let
$R_1,\ldots,R_n$ denote the \relu{} constraints in $N'$. For each
\relu{} $R_i$, we use a Boolean variable $r_i$ to indicate whether the
constraint is active ($r_i$ is true), or inactive ($\neg r_i$ is true). We then define
$\Gamma$ to be a CNF formula over these Boolean variables:
\[
  \Gamma :\qquad \bigwedge( \bigvee_{l_j\in\bigcup_{i=1}^n\{r_i,\neg r_i\}}  l_j)
\]

In order for our approach to maintain soundness, $\Gamma$ needs to be
a \emph{valid formula} for $\langle N'', \varphi\rangle$; i.e., if
there exists an assignment that satisfies
$\langle N'', \varphi\rangle$, it must also satisfy $\Gamma$.  Under
this assumption, a verifier can use $\Gamma$ to avoid case-splitting
during the verification of the refined network, using
unit-propagation~\cite{BiHeVa09}. For example, suppose that one of the
clauses in $\Gamma$ is $(r_1\lor \neg r_2 \lor \neg r_3)$, and that
while verifying the refined network, the verifier has performed two
case splits already to the effect that $r_1$ is false ($R_1$ is
inactive) and $r_2$ is true ($R_2$ is active). In this case, the
verifier can immediately set $r_3$ to false, as it is guaranteed that
no satisfying assignments exist where $r_3$ is true, as these would
violate the clause above. This guarantees that no future splitting is
performed on $R_3$.


More formally, we state the following Lemma:
\begin{lemma} [Soundness of Residual Reasoning]
  Let $\langle N', \varphi\rangle$ and $\langle N'', \varphi\rangle$
  be verification queries on an abstract network $N'$ and its
  refinement $N''$, being
  solved by a sound verifier; and let $\Gamma$ be a valid formula as
  described above. If the verifier uses $\Gamma$ to deduce the phases
  of \relu{} constraints using unit propagation for the verification of $\langle N'', \varphi\rangle$, soundness is
  maintained.
\end{lemma}
The proof is straightforward, and is omitted. We also note that when
multiple consecutive refinement steps are performed, some renaming of
variables within $\Gamma$ is required; we discuss this in later
sections.

\section{Residual Reasoning and Neuron-Merging Abstraction}
\label{sec:rr_instantiation}
Our proposed approach for residual reasoning is quite general; and our
definitions do not specify how $\Gamma$ should be populated. In order
to construct in $\Gamma$ a lemma that will be valid for future
refinements of the network, one must take into account the specifics
of the abstraction-refinement scheme in use. In this section, we
propose one possible integration with a recently proposed
abstraction-refinement scheme that merges and splits
neurons~\cite{ElGoKa20}, which was discussed in Sec.~\ref{sec:preliminaries}.

We begin by revisiting our example from Fig.~\ref{fig:running-example}.
Suppose that in order to solve query
$\langle N, \varphi\rangle$, we generate the abstract network $N'$
and attempt to verify  $\langle N', \varphi \rangle$ instead.  During
verification, some case splits are performed; and it is discovered
that when neuron $v_{2,1+2}$'s \relu{} function is active, no
satisfying assignment can be found. Later, the verifier discovers a
satisfying assignment for which $v_{2,1+2}$ is inactive:
$\vec{x}= \langle 0, 1\rangle \Rightarrow N'(\vec{x})=16 > 14
$. Unfortunately, this counterexample turns out to be spurious,
because
$ N(\langle 0,1\rangle )=9\leq14 $, and so the network is refined:
node $ v_{2,1+2} $ is split into two new nodes, ($ v_{2,1},v_{2,2} $),
giving rise to the refined network $ N'' $. The verifier then
begins solving query $\langle N'', \varphi\rangle$.

We make the following claim: because no satisfying assignment exists
for $\langle N', \varphi\rangle$ when $v_{2,1+2}$ is active, and
because $v_{2,1+2}$ was refined into ($ v_{2,1},v_{2,2} $), then no
satisfying assignment exists for $\langle N'', \varphi\rangle$
when $ v_{2,1} $ and $ v_{2,2} $ are both active. In other words, it
is sound to verify $\langle N'', \varphi\rangle$ given
$\Gamma=(\neg r_{2,1}\vee \neg r_{2,2})$, where $r_{2,1}$ and $r_{2,2}$
correspond to the activation phase of $v_{2,1}$ and $v_{2,2} $,
respectively. Thus, e.g., if the verifier performs a case split and fixes
$v_{2,1}$ to its active phase, it can immediately set $v_{2,2}$  to
inactive, without bothering to explore the case where $v_{2,2}$ is
also active.

In order to provide intuition as to why this claim holds, we now
formally prove it; i.e., we show that if an input $\vec{x}$ satisfies
$\langle N'', \varphi\rangle $ when $v_{2,1}$ and $v_{2,2}$ are both
active, then it must also satisfy $\langle N'_1, \varphi\rangle$ when
$v_{2,1+2}$ is active. First, we observe that because $N''$ is a
refinement of $N'$, it
immediately follows that $ N''(\vec{x})\leq N'(\vec{x})$;
and because the property $\varphi$ is of the form $y>c$, if
$\langle N'', \varphi\rangle $ is \sat{} then
$\langle N', \varphi\rangle $ is also \sat{}. Next, we observe that
$N''$ and $N'$ are identical in all layers preceding
$v_{2,1} , v_{2,2}$ and $v_{2,1+2}$, and so all neurons feedings into
these three neurons are assigned the same values in both
networks. Finally, we assume towards contradiction that $v_{2,1+2}$ is
not active; i.e., that $3\cdot \relu{}(v_{1,1})-\relu{}(v_{1,3})<0$;
but because it also holds that
$ v_{2,1}=3\cdot \relu{}(v_{1,1})-\relu{}(v_{1,3}) $, this contradicts
the assumption that $v_{2,1}$ and $v_{2,2}$ are both active. This
concludes our proof, and shows that
$\Gamma=(\neg r_{2,1}\vee \neg r_{2,2})$ is valid.

In the remainder of this section, we formalize the principle
demonstrated in the example above. The formalization is complex, and
relies on the details of the abstraction mechanism~\cite{ElGoKa20}; we give here the gist of the formalization, with
additional details appearing in
Appendix~\ref{appendix:formalization}.

Using the terminology of~\cite{ElGoKa20}, two nodes can be merged as
part of the abstraction process if they share a \emph{type}:
specifically, if they are both \inc{} neurons, or if they are both
\dec{} neurons. An \inc{} neuron has the property that \emph{increasing} its
value results in an increase to the network's single output; whereas a
\dec{} neuron has the property that \emph{decreasing} its value
increases the network's single output. In our running example, neuron $ v_{2,1+2} $ is an \inc{}
neuron, whereas neuron $ v_{2,3} $ is a \dec{} neuron.

We use the term \textit{abstract neuron} to refer to a neuron
generated by the merging of two neurons from the same category, and
the term \emph{refined neuron} to refer to a neuron that was
generated (restored) during a refinement step.
An example for the merging of two \inc{} neurons appears in
Fig.~\ref{fig:abstract-and-refined-networks}.

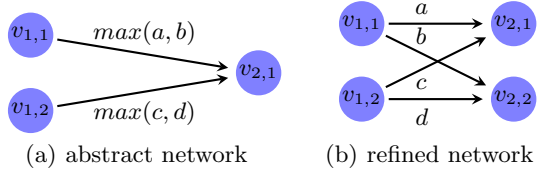
\begin{figure}[htp]
  \centering
  \subfigure[abstract network]
  {
    \label{fig:abstractnetwork}
    \scalebox{1} {
      \def\layersep{2.5cm}
      \begin{tikzpicture}[shorten >=1pt,->,draw=black!50, node distance=\layersep,font=\footnotesize]
        
        \node[hidden neuron] (I-1) at (0, -0.5) {$v_{1,1}$};
        \node[hidden neuron] (I-2) at (0, -1.5) {$v_{1,2}$};
        
        \node[hidden neuron] (H-1) at (3, -1.0) {$v_{2,1}$};
        
        \draw[nnedge] (I-1) -- node[above] {$max(a,b)$} (H-1);
        \draw[nnedge] (I-2) -- node[below] {$max(c,d)$} (H-1);
      \end{tikzpicture}
    }
  }
  \subfigure[refined network]
  {
  	\label{fig:refinednetwork}
  	\scalebox{1} {
  		\def\layersep{2.5cm}
  		\begin{tikzpicture}[shorten >=1pt,->,draw=black!50, node distance=\layersep,font=\footnotesize]
  			
  			\node[hidden neuron] (I-1) at (0, -0.5) {$v_{1,1}$};
  			\node[hidden neuron] (I-2) at (0, -1.5) {$v_{1,2}$};
  			
  			\node[hidden neuron] (H-1) at (2, -0.5) {$v_{2,1}$};
  			\node[hidden neuron] (H-2) at (2, -1.5) {$v_{2,2}$};
  			
  			\draw[nnedge] (I-1) -- node[above left] {$a$} (H-1);
  			\draw[nnedge] (I-1) -- node[above, pos=0.35] {$b$} (H-2);
  			\draw[nnedge] (I-2) -- node[below, pos=0.35] {$c$} (H-1);
  			\draw[nnedge] (I-2) -- node[below left] {$d$} (H-2);					
  		\end{tikzpicture}
  	}
  }
  \caption{Abstraction/refinement of two \textit{inc} neurons.}
  \label{fig:abstract-and-refined-networks}
\end{figure}	

We now state our main theorem, which justifies our method of
populating $\Gamma$. We then give an outline of the proof, and refer
the reader to Theorem 2 in Appendix~\ref{appendix:formalization} for additional details.

\begin{theorem}
  \label{theorem:prune}
  Let $\langle N, \varphi\rangle$ be a verification query, where
  $N:\vec{x}\rightarrow y$ has a single output node $y$, and $\varphi$
  is of the form $\varphi= (\vec{l}\leq\vec{x}\leq\vec{u})\wedge
  (y>c)$. Let $N'$ be an abstract network obtained from $N$ using
  neuron merging, and let $N''$ be a network obtained from
  $N'$ using a single refinement step in reverse order of abstraction. Specifically, let $v$ be
  a neuron in $N'$ that was split into two neurons $v_1,v_2$ in
  $N''$. Then, if a certain guard condition $\guard{}$ holds, we
  have the following:
  \begin{enumerate}
  \item If $ v $ is \inc{} neuron, and during the verification of $\langle N', \varphi\rangle$, the verifier determines that setting $v$ to active leads to an \unsat{} branch of the search tree, then $\Gamma=(\neg r_1\lor \neg r_2)$ is a valid formula for $\langle N'', \varphi\rangle$ (where $r_1$ and $r_2$ correspond to $v_1$ and $v_2$, respectively).
  \item Symmetrically, if setting a \dec{} neuron $v$ to inactive
    leads to an \unsat{} branch, then $\Gamma=(r_1\lor r_2)$ is a
    valid formula for $\langle N'', \varphi\rangle$.
  \end{enumerate}
\end{theorem}

The guard condition $\guard{}$ is intuitively defined as the
conjunction of the following stipulations, whose goal is to enforce
that the branches in both search trees (of
$\langle N', \varphi\rangle$ and $\langle N'', \varphi\rangle$) are
sufficiently similar:
\begin{enumerate}
\item The same case splits have been applied during the verification
  of $ N' $ and $ N'' $, for all neurons in the preceding layers of
  the abstract neuron and for any other neurons in the same layer as
  the abstract neuron.
\item The same case splits have been applied during the verification of $ N' $ and $ N'' $ for the abstract neuron and its refined neurons.
\item Every \inc{} neuron in layers following the layer of
  $v,v_1,v_2$
  has been split on and set to active, and every \dec{} neuron in these
  layers has been split on and set to inactive.  
\end{enumerate}
We stress that the guard condition $\guard{}$ does not change the way
$\Gamma$ is populated; but that the verifier must ensure that
$\guard{}$ holds before it applies unit-propagation based on $\Gamma$.
The precise definitions and proof appear in
Appendix~\ref{appendix:formalization}.

When the conditions of the theorem are met, a satisfying assignment
within the specific branch of the search tree of the refined network
would indicate that the corresponding branch in the abstract network
is also \sat{}, which we already know is untrue; and consequently,
that branch can be soundly skipped. To prove the theorem, we require the two
following lemmas, each corresponding to one of the two cases of the
theorem.

\renewcommand*{\proofname}{Proof Outline}

\begin{lemma}
  \label{abstract max negative->refinement negative}
  Given an input $\vec{x}$, if the value of an abstract \inc{} node
  $v$ is negative, then at least one of the values of the refined
  nodes $v_1$ and $v_2$ is negative for the same $\vec{x}$.
\end{lemma}
\begin{proof}
We explain how to prove the lemma using the general network from
Fig.~\ref{fig:abstract-and-refined-networks}; and this proof can be
generalized to any network in a straightforward way.
  Observe nodes $v_{2,1}$ and $v_{2,2}$ in
  Fig.~\ref{fig:refinednetwork}, which are nodes refined from node
  $v_{2,1}$ in Fig~\ref{fig:abstractnetwork}.
  We need to prove that the following
  implication holds:
  \[
    x_1 \cdot max(a,b) + x_2 \cdot max(c,d) < 0 \Rightarrow (x_1 \cdot
    a + x_2 \cdot c < 0 \lor x_1 \cdot b + x_2 \cdot d < 0)
    \]
     The values of $ x_1,x_2 $ are the outputs of \relu{}s, and so are
     non-negative. We can thus split 
     into 4 cases:
  \begin{enumerate}
  \item If $ x_1=0, x_2=0 $, the implication holds trivially.
  \item If $ x_1=0, x_2>0 $, then
    $x_2 \cdot max(c,d)<0 $, and  so $ c,d<0 $. We get that $ x_1 \cdot
    a+x_2 \cdot c=x_2 \cdot c<0 $ and $ x_1 \cdot b+x_2 \cdot d=x_2
    \cdot d<0 $,  and so the implication holds.
  \item The case where $ x_1>0, x_2=0 $ is symmetrical to the previous
    case.
  \item If $ x_1>0, x_2>0 $, the implication becomes
    \[
      max(x_1 \cdot a,x_1 \cdot b) + max(x_2 \cdot c,x_2 \cdot d) < 0 \Rightarrow (x_1 \cdot
      a + x_2 \cdot c < 0 \lor x_1 \cdot b + x_2 \cdot d < 0)
    \]
  Let us  denote $ a'=x_1 \cdot a,\ b'=x_1\cdot b$ and $c'=x_2\cdot
  c,\ d'=x_2\cdot d $. The lemma then becomes:
  \[
    max(a',b')+max(c',d') < 0 \Rightarrow a'+c'<0 \lor b'+d'<0
  \]
  
  \begin{itemize}
  	\item If $ a' \ge b' $, then $ a'=max(a',b') $ and $
          a'+max(c',d') < 0 $. We then get that
          \[
            b'+d' \le a' + d' \le a'+max(c',d') < 0
          \]
          as needed.
  	\item If $ a' < b' $, then $ b'=max(a',b') $ and $
          b'+max(c',d') < 0 $. We then get that
          \[
            a'+c' \le b' + max(c',d') < 0
          \]
          again as needed.
  \end{itemize}
  \end{enumerate}
\end{proof}

Lemma~\ref{abstract max negative->refinement negative} establishes the
correctness of Theorem~\ref{theorem:prune} for \inc{} neurons. We also
have the following, symmetrical lemma for \dec{} neurons:

\renewcommand*{\proofname}{Proof}

\begin{lemma}
	\label{abstract min positive->refinement positive}
	Given an input $\vec{x}$, if the value of an abstract \dec{}
        node $v$ is positive, then at least one of the values of the
        refined nodes $v_1$ and $v_2$ is positive for the same $\vec{x}$.
\end{lemma}

The proof outline is similar to that of Lemma~\ref{abstract max
  negative->refinement negative}, and appears in
Appendix~\ref{appendix:formalization}.

The result of applying Theorem~\ref{theorem:prune} as part of the
verification of our running example from
Fig.~\ref{fig:running-example} is illustrated in
Fig.~\ref{fig:AR4-pruning}. There, each rectangle represents a single
verification query, and blue lines indicate abstraction steps. Within
each rectangle, we see the verifier's search tree, where triangles
signify sub-trees --- and red triangles are sub-trees where the
verifier was able to deduce that no satisfying assignment exists. The
figure shows that, when solving the query in the bottom rectangle, the
verifier discovered an \unsat{} sub-tree that meets the conditions of
the Theorem. This allows the verifier to deduce that another sub-tree,
in another rectangle/query, is also \unsat{}, as indicated by a green
arrow.  Specifically, by discovering that setting $v_{2,1+2}$ to
\emph{active} results in \unsat{}, the verifier can deduce that
setting $v_{2,1}$ to \emph{active} and then $v_{2,2}$ to \emph{active}
must also result in \unsat{}.

\begin{figure}
  \begin{center}
  	\includegraphics[scale=0.3]{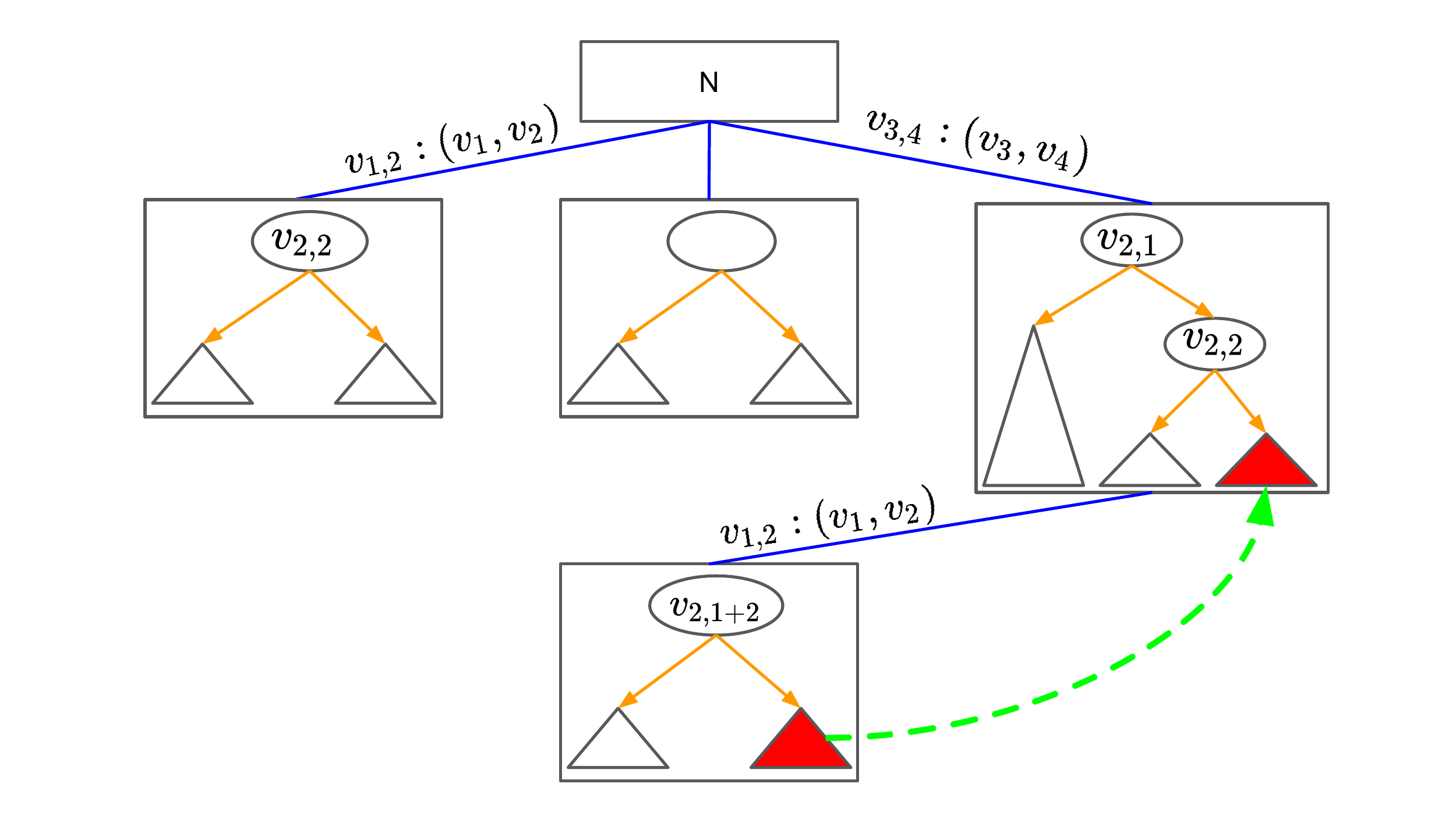}
  \end{center}
  \caption{Applying Theorem~\ref{theorem:prune} while solving the
    query from Fig.~\ref{fig:running-example}.}
  \label{fig:AR4-pruning}
\end{figure}

\mysubsection{Multiple Refinement Steps.}  So far, we have only
discussed populating $\Gamma$ for a single refinement step. However,
$\Gamma$ can be adjusted as multiple refinement steps are
performed. In that case, each invocation of
Theorem~\ref{theorem:prune} adds another CNF clause to the formula
already stored in $\Gamma$. Further, some book-keeping and renaming is
required, as neuron identifiers change across the different networks:
intuitively, whenever an abstract neuron $v$ is split into neurons
$v_1$ and $v_2$, the literal $v$ must be replaced with $v_1\vee
v_2$. These notions are formalized in Sec.~\ref{sec:ar4}; and the
soundness of this procedure can be proven using repeated invocations
of Theorem~\ref{theorem:prune}.

\section{Adding Residual Reasoning to Reluplex}
\label{sec:ar4}
Unlike in previous abstraction-refinement approaches for DNN
verification~\cite{ElGoKa20, AsHaKrMu20, PrAf20}, residual
reasoning requires instrumenting the DNN verifier in question, for the
purpose of populating, and using, $\Gamma$. We next describe such an
instrumentation for the Reluplex algorithm~\cite{Katz2017Reluplex},
which is the core algorithm used in the state-of-the-art verifier
Marabou~\cite{Katz2019Marabou}.  Reluplex, a sound and complete DNN
verification algorithm, employs case-splitting as discussed in
Sec.~\ref{sec:preliminaries}, along with various heuristics for
curtailing the search space and reducing the number of
splits~\cite{WuOzZeIrJuGoFoKaPaBa20, WuZeKaBa22}; and it has been
integrated with abstraction-refinement techniques
before~\cite{ElGoKa20}, rendering it a prime candidate for residual
reasoning. We term our enhanced version of Reluplex $AR^4$, which
stands for Abstraction-Refinement with Residual Reasoning for
Reluplex.

For our purposes, it is convenient to think of Reluplex as a set of
derivation rules, applied according to an implementation-specific
strategy. The most relevant parts of this calculus, borrowed from Katz
et al.~\cite{Katz2017Reluplex} and simplified, appear in
Fig.~\ref{fig:Reluplex-calculus}; other rules, specifically those that
deal with the technical aspects of solving linear problems, are
omitted for brevity.

\begin{figure}[ht!]
	\begin{centering}
		\scriptsize

		\failure
		\drule{
                  \exists x\in\allvars.\ l(x)>u(x)
                  }
		{
			\unsat{}
		}
                \qquad
		\reluSplit
		\drule{
			\langle x_i, x_j\rangle \in \reluSet, 
			\ \ 
			\lb(x_i)<0,
			\ \ 
			\ub(x_i)>0
		}
		{
			\ub(x_i):= 0 
			\qquad
			\lb(x_i):= 0 
		}
		\medskip

		\reluSuccess
		\drule{
			\forall x\in\allvars. \ 
			\lb(x) \leq \assignment(x) \leq \ub(x), 
			\ \ 
			\forall \langle x,y \rangle \in \reluSet. \
			\assignment(y) = \max{}(0, \assignment(x))
		}
		{
			\sat{}
		}
		\medskip

		\caption{Derivation rules of Reluplex calculus
                  (partial, simplified).}
		\label{fig:Reluplex-calculus}
	\end{centering}
\end{figure}

Internally, Reluplex represents the verification query as a set of
linear equalities and lower/upper bounds over a set of variables, and
a separate set of \relu{} constraints.  A \textit{configuration} of
Reluplex over a set of variables $\allvars$ is either a distinguished
symbol from the set $\{\sat{},\unsat{}\}$, or a tuple
$\langle T, \lb, \ub, \assignment, \reluSet \rangle$, where: $T$, the
\emph{tableau}, contains the set of linear equations; $\lb, \ub$ are
mappings that assign each variable $x\in\allvars$ a lower and an upper
bound, respectively; $\assignment$, the \emph{assignment}, maps each
variable $x\in\allvars$ to a real value; and $\reluSet$ is the set of
\relu{} constraints, i.e. $\langle x, y\rangle\in\reluSet$ indicates
that $y=\relu{}(x)$. Reluplex will often derive \emph{tighter bounds}
as it solves a query; i.e., will discover greater lower bounds or
smaller upper bounds for some of the variables.

Using these definitions, the rules in Fig.~\ref{fig:Reluplex-calculus}
can be interpreted follows: \failure{} is applicable when Reluplex
discovers inconsistent bounds for a variable, indicating that the
query is \unsat{}.  \reluSplit{} is applicable for any \relu{}
constraint whose linear phase is unknown; and it allows Reluplex to
``guess'' a linear phase for that \relu{}, by either setting the upper
bound of its input to $0$ (the inactive case), or the lower bound of
its input to $0$ (the active case).  \reluSuccess{} is applicable when
the current configuration satisfies every constraint, and returns \sat{}.

In order to support \arf{}, we extend the Reluplex calculus with
additional rules, depicted in Fig.~\ref{fig:ar4-calculus}. We use the
context variable $\Gamma$, as before, to store a valid CNF formula to
assist the verifier; and we also introduce two additional context
variables, $ \Gamma_A$ and $\Gamma_B$, for book-keeping
purposes. Specifically, $ \Gamma_A $ stores a mapping between abstract
neurons and their refined neurons; i.e., it is comprised of triples
$ \langle v,v_1,v_2\rangle $, indicating that abstract neuron $v$ has
been refined into neurons $v_1$ and $v_2$.  $ \Gamma_B $ is used for
storing past case splits performed by the verifier, to be used in
populating $\Gamma$ when the verifier finds an \unsat{} branch.
 Given variable $ x $ of neuron $ v $, we use $\guard^{\inc}(\Gamma_A,\Gamma_B,x)$ and $\guard^{\dec}(\Gamma_A,\Gamma_B,x)$ to denote a Boolean
 function that returns true if and only if the guard conditions
 required for applying Theorem~\ref{theorem:prune} hold, for an \inc{}
 or \dec{} neuron $ v $, respectively.

\begin{figure}[ht!]
  \begin{centering}
    \scriptsize
    
	\reluSplit
    \drule{
    	\langle x_i, x_j\rangle \in \reluSet, 
    	\ \ 
    	\lb(x_i)<0,
    	\ \ 
    	\ub(x_i)>0
    }
    {
    	\ub(x_i):= 0 \qquad\qquad \lb(x_i):= 0
    	\\
    	\qquad
    	\Gamma_B:=\Gamma_B\lor r_i
    	\qquad 
    	\Gamma_B:=\Gamma_B\lor \neg r_i
    }
    \medskip
    
    \failure
    \drule{
      \exists x_i\in\allvars.\ l(x_i)>u(x_i)
    }
    {
      \unsat{}, \Gamma:=\Gamma \land \Gamma_B
    } \qquad
    \StepAbstract
    \drule{
      CanAbstract(x_1,x_2)
    }
    {
      \Gamma_A:=\Gamma_A \cup \langle x_{1,2}, x_1,x_2\rangle
    }
    \medskip
    
    \StepRefine
    \drule{
    \Gamma_A \neq \emptyset
    }
    {
    	\Gamma_A := \Gamma_A[:-1]
    }
    \qquad
    \realsuccess
    \drule{
      \sat \land isRealSAT(\Gamma_A)
    }
    {
      \realsat
    }
    
    \medskip 

    \ApplyAbstraction
	\drule{
          true
	}
	{
		Abstract(\Gamma_A), UpdateContext(\Gamma,\Gamma_A,\Gamma_B)
	}
	\medskip
    
	\Prune{1}
	\drule{
		\langle x,x_i,x_j\rangle \in \Gamma_A \land \neg r_i,\neg r_j\in\Gamma_B \land \guard^\inc(\Gamma_A,\Gamma_B,x) \land l(x_i)=0
	}
	{
		u(x_j)=0, \Gamma_B:=\Gamma_B \lor r_j
	}
        
    \medskip 
    
	\Prune{2}
	\drule{
		\langle x,x_i,x_j\rangle \in \Gamma_A \land r_i,r_j\in\Gamma_B \land \guard^\dec(\Gamma_A,\Gamma_B,x) \land u(x_i)=0
	}
	{
		l(x_j)=0, \Gamma_B:=\Gamma_B \lor \neg r_j
	}
	
    \medskip 
        
    \caption{Derivation rules for the $ AR^4 $ calculus.}
    \label{fig:ar4-calculus}
  \end{centering}
\end{figure}

The rules in Fig.~\ref{fig:ar4-calculus} are interpreted as follows.
\StepAbstract is used for merging neurons and creating the
initial, abstract network.
\StepRefine is applicable when dealing with an abstract network (indicated by $\Gamma_A\neq\emptyset$), and performs a refinement step by canceling the last abstraction step. 
\ApplyAbstraction is applicable anytime, and generates an abstract
network according to the information in $\Gamma_A$, updating the relevant contexts correspondingly. 
The \reluSuccess rule from the original Reluplex calculus in included, as is, in the \arf{} calculus; but we note that a \sat{} conclusion that it reaches is applicable only to the current, potentially abstract network, and could thus be spurious. To solve this issue, we add the \realsuccess rule, which checks whether a \sat{} result is true for
the original network as well. Thus, in addition to \sat{} or \unsat{},
the \realsat{} state is also a terminal state for our calculus.

The \failure rule replaces the Reluplex rule with the same name, and
is applicable when contradictory bounds are discovered; but apart from
declaring \unsat{}, it also populates $\Gamma$ with the current
case-split history in $\Gamma_B$, for future pruning of the search
space.  The \reluSplit rule, similarly to the Reluplex version,
guesses a linear phase for one the \relu{}s, but now also records that
action in $\Gamma_B$. Finally, the \Prune{1/2} rules are applicable
when all the conditions of Theorem~\ref{theorem:prune} (for the
\inc{}/\dec{} cases, respectively) are met, and they trim the search tree and update $\Gamma$ accordingly.

\mysubsection{Side Procedures.}  We intuitively describe the
four functions, $ CanAbstract $, $Abstract$, $ UpdateContext $ and
$isRealSat $, which appear in the calculus; additional details can be
found in Appendix~\ref{appendix:ar4-side-procedures}.

\begin{itemize}
\item CanAbstract (Algorithm~\ref{alg:can_abstract},
  Appendix~\ref{appendix:ar4-side-procedures}) checks whether two
  neurons can be merged according to their types; and
  also checks whether the assignment to the variables did not
  change yet during the verification process.
  
\item Abstract (Algorithm~\ref{alg:abstractSequence},
  Appendix~\ref{appendix:ar4-side-procedures}) performs multiple
  abstraction steps. A single abstraction step
  (Algorithm~\ref{alg:abstract step},
  Appendix~\ref{appendix:ar4-side-procedures}) is defined as the
  merging of two neurons of the same type, given that the assignments
  to their variables were not yet changed during verification.
  
\item UpdateContext clears the case-splitting context by setting
  $(\Gamma_B=\emptyset)$, and also updates clauses in $\Gamma$ to use new
  variables: for variables representing \inc{} nodes, $\neg r$ is replaced with $\neg r_1\lor \neg r_2$; and for variables representing \dec{} nodes, $r$ is replaced with $r_1\lor r_2$.
  
\item $ isRealSat $ (Algorithm~\ref{alg:isRealSAT},
  Appendix~\ref{appendix:ar4-side-procedures}) checks whether a
  counterexample holds in the original network.
  
\end{itemize}
	
\mysubsection{Implementation Strategy.} 
The derivation rules in Fig.~\ref{fig:ar4-calculus} define the ``legal
moves'' of \arf{} --- i.e., we are
guaranteed that by applying them, the resulting verifier will be
sound. We now discuss one possible \emph{strategy} for
applying them, which we used in our proof-of-concept
implementation.

We begin by applying \StepAbstract{} to saturation, in
order to reach a small abstract network; and then apply once the
\ApplyAbstraction{} rule, to properly initialize the context variables. 
Then, we enter into the loop of abstraction-based verification: we apply
the Reluplex core rules using existing
strategies~\cite{Katz2019Marabou}, but every time the \reluSplit{}
rule is applied we immediately apply \Prune{1} and \Prune{2}, if they
are applicable.  The \failure{} and \reluSuccess rules are applied as
in Reluplex, and \realsuccess is applied immediately after
\reluSuccess if it is applicable; otherwise, we apply \StepRefine{},
and repeat the process.  We also attempt to apply \Prune{1} and
\Prune{2} after each application of \failure, since it updates
$\Gamma$.
     
\section{Experiments and Evaluation}\label{sec:evaluation}

For evaluation purposes, we created a proof-of-concept implementation
of \arf{}, and compared it to the only tool currently available that
supports CEGAR-based DNN verification --- namely, the extension of
Marabou proposed in~\cite{ElGoKa20}. We used both tools to verify a
collection of properties of the ACAS Xu family of 45 DNNs
(each with 310 neurons, spread across 8 layers) for airborne
collision avoidance~\cite{JuLoBrOwKo2016}. Specifically, we verified a
collection of 4 safety properties and 20 adversarial robustness
properties for these networks, giving a total of 1080 benchmarks; and
from these experiments we collected, for each tool, the runtime
(including instrumentation time), the number of properties
successfully verified within the allotted timeout of two hours, and
the number of case splits performed.  The experiments were conducted
on x86-64 Gnu/Linux based machines using a single Intel(R) Xeon(R)
Gold 6130 CPU @ 2.10GHz core.  Our code is publicly available
online.\footnote{\url{https://drive.google.com/file/d/1onk3dW3yJeyXw8_rcL6wUsVFC1bMYvjL}}

The results of our experiments appear in
Table~\ref{Fig:table-experiments_summary} and Fig.~\ref{fig:graph-experiments_summary}, and demonstrate the
advantages of \arf{} compared to AR. \arf{} timed out on 18.4\% fewer benchmarks,
and solved 188 benchmarks more quickly than AR, compared
to 119 where AR was faster.
We note that in these comparisons, we treated experiments
 in which both tools finished within 5 seconds of each other as ties.
Next, we observe that residual reasoning successfully curtailed the
search space: on average, \arf{} traversed 5.634 states of the search
tree per experiment, compared to 7.178 states traversed by AR --- a
21.5\% decrease.

     \begin{figure}[!ht]
      \centering
      \scalebox{0.365}{
   \begin{tabular}{cc}
     \includegraphics[]{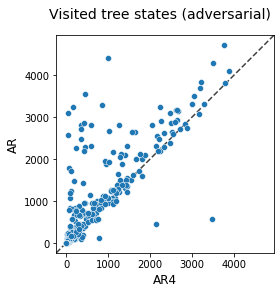} &  
     \includegraphics[]{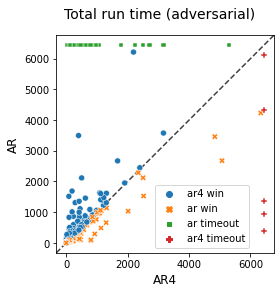}  \\
     \includegraphics[]{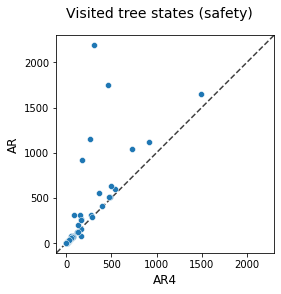} &
     \includegraphics[]{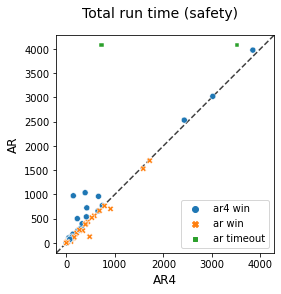} \\
   \end{tabular}
   }
   \caption{Comparing $ AR^4 $ and $ AR $.}
   \label{fig:graph-experiments_summary}
 \end{figure}

\begin{table}[t]
  \caption{Comparing $ AR^4 $ and $ AR $.}
        \label{Fig:table-experiments_summary}
	\centering
        \resizebox{\columnwidth}{!}{%
	\begin{tabular}{ | C{2cm}c | cC{1.5cm} | C{1.5cm}c | cC{1.5cm} | C{1.5cm}c | cC{1.5cm} | C{1.5cm}c | }
            \toprule
          &&& \multicolumn{2}{c}{Adversarial}
          &&& \multicolumn{2}{c}{Safety}
          &&& \multicolumn{2}{c}{Total (Weighted)} & \\

          \cline{4-5}\cline{8-9}\cline{12-13} 
          &&& $ AR^4 $ & AR &&& $ AR^4 $ & AR &&& $ AR^4 $ & AR & \\

          \midrule
          
		Timeouts &&& 95/900 & 116/900 &&& 7/180 & 9/180 &&&
                                                                    102/1080 & 125/1080 & \\ \hline 
		Instances solved more quickly &&& 160 & 95 &&& 28 & 24 &&& 188 & 119 & \\ \hline
		Uniquely solved &&& 26 & 5 &&& 2 & 0 &&& 28 & 5 & \\ \hline
	     	Visited tree states &&& 6.078 & 7.65 &&& 3.569 & 4.98 &&&
                                                                      5.634
                       & 7.178 & \\ \hline
	        Avg. instrumentation time &&& 91.54 & - &&& 36.5 & -
          &&& 82.367 & - & \\ 

          \bottomrule
	\end{tabular}
      }
    \end{table}
 
Despite the advantages it often affords, \arf{} is not always
superior to AR --- because the cost of instrumenting the verifier is
not always negligible.
In our
 experimenters, the verifier spent an average of 82 seconds executing
 our instrumentation code out of an average total runtime of 885
 seconds --- nearly 10\%, which is quite significant.
 In order to mitigate this issue, moving forward we plan to strengthen
 the engineering of our tool, e.g., by improve its implementation of
 unit-propagation through the use of watch literals~\cite{BiHeVa09}.

 \section{Related Work}\label{sec:related-work} 
 Modern DNN verification schemes leverage principles from
 SAT and SMT solving~\cite{Katz2019Marabou,
   Katz2017Reluplex, HuKwWaWu17, Ehlers2017Planet, NaKaRySaWa17},
 mixed integer linear programming~\cite{TjXiTe17, BuTu2017,
   DuJhSaTi18, Ehlers2017Planet}, abstract
 interpretation~\cite{GeMiDrTsChVe18, WaZhXuLiJaHsKo21, SiGePuVe19,
   MuMaSiPuVe22}, and others. Many of these approaches apply
 case-splitting, and could benefit from residual
 reasoning.

 Abstraction-refinement techniques are known to be highly beneficial
 in verifying hand-crafted systems~\cite{ClGrJhLuVe00CEGAR}, and recently
 there have been promising attempts to apply them to DNN verification
 as well~\cite{ElGoKa20, AsHaKrMu20, PrAf20}. As far as we know,
 ours is the first attempt to apply residual reasoning in this context.
 
 \section{Conclusion}\label{sec:conclusion}
 As DNNs are becoming increasingly integrated into safety-critical
 systems, improving the scalability of DNN verification is
 crucial. Abstraction-refinement techniques could play a significant
 part in this effort, but they can sometimes create redundant work for
 the verifier. The residual reasoning technique that we propose can
 eliminate some of this redundancy, resulting in a speedier
 verification procedure. We regard our work here as another step
 towards tapping the potential of abstraction-refinement methods in
 DNN verification.

 Moving forward, we plan to improve the
 engineering of our \arf{} tool; and to integrate it with other
 abstraction-refinement DNN verification techniques~\cite{AsHaKrMu20}.

 \medskip
 \noindent
 \textbf{Acknowledgments.}  This work was supported by ISF grant 683/18.  We thank Jiaxiang Liu and
 Yunhan Xing for their insightful comments about this work.

 {
   \newpage

\begin{thebibliography}{10}

\bibitem{ANGELOV2020185}
P.~Angelov and E.~Soares.
\newblock {Towards explainable deep neural networks (xDNN)}.
\newblock {\em Neural Networks}, 130:185--194, 2020.

\bibitem{AsHaKrMu20}
P.~Ashok, V.~Hashemi, J.~Kretinsky, and S.~M\"{u}hlberger.
\newblock {DeepAbstract: Neural Network Abstraction for Accelerating
  Verification}.
\newblock In {\em Proc. 18th Int. Symposium on Automated Technology for
  Verification and Analysis (ATVA)}, pages 92--107, 2020.

\bibitem{AzCoPa20}
S.~Azzopardi, C.~Colombo, and G.~Pace.
\newblock {A Technique for Automata-based Verification with Residual
  Reasoning}.
\newblock In {\em Proc. 8th Int. Conf. on Model-Driven Engineering and Software
  Development (MODELSWARD)}, pages 237--248, 2020.

\bibitem{BaLiJo21}
S.~Bak, C.~Liu, and T.~Johnson.
\newblock {The Second International Verification of Neural Networks Competition
  (VNN-COMP 2021): Summary and Results}, 2021.
\newblock Technical Report. \url{http://arxiv.org/abs/2109.00498}.

\bibitem{BiHeVa09}
A.~Biere, M.~Heule, and H.~van Maaren.
\newblock {\em {Handbook of Satisfiability}}.
\newblock IOS Press, 2009.

\bibitem{BoDeDwFiFlGoJaMoMuZhZhZhZi16}
M.~Bojarski, D.~Del~Testa, D.~Dworakowski, B.~Firner, B.~Flepp, P.~Goyal,
  L.~Jackel, M.~Monfort, U.~Muller, J.~Zhang, X.~Zhang, J.~Zhao, and K.~Zieba.
\newblock {End to End Learning for Self-Driving Cars}, 2016.
\newblock Technical Report. \url{http://arxiv.org/abs/1604.07316}.

\bibitem{BuTu2017}
R.~Bunel, I.~Turkaslan, P.~Torr, P.~Kohli, and M.~Kumar.
\newblock {Piecewise Linear Neural Network verification: A Comparative Study},
  2017.
\newblock Technical Report. \url{http://arxiv.org/abs/1711.00455}.

\bibitem{ClGrJhLuVe00CEGAR}
E.~Clarke, O.~Grumberg, S.~Jha, Y.~Lu, and H.~Veith.
\newblock {Counterexample-Guided Abstraction Refinement}.
\newblock In {\em Proc. 12th Int. Conf. on Computer Aided Verification (CAV)},
  pages 154--169, 2000.

\bibitem{Dantzig1963}
G.~Dantzig.
\newblock {\em {Linear Programming and Extensions}}.
\newblock Princeton University Press, 1963.

\bibitem{devlin-etal-2019-bert}
J.~Devlin, M.-W. Chang, K.~Lee, and K.~Toutanova.
\newblock {BERT: Pre-training of Deep Bidirectional Transformers for Language
  Understanding}, 2018.
\newblock Technical Report. \url{http://arxiv.org/abs/1810.04805}.

\bibitem{DuJhSaTi18}
S.~Dutta, S.~Jha, S.~Sanakaranarayanan, and A.~Tiwari.
\newblock {Output Range Analysis for Deep Neural Networks}.
\newblock In {\em Proc. 10th NASA Formal Methods Symposium (NFM)}, pages
  121--138, 2018.

\bibitem{Ehlers2017Planet}
R.~Ehlers.
\newblock {Formal Verification of Piece-Wise Linear Feed-Forward Neural
  Networks}.
\newblock In {\em Proc. 15th Int. Symp. on Automated Technology for
  Verification and Analysis (ATVA)}, pages 269--286, 2017.

\bibitem{ElGoKa20}
Y.~Elboher, J.~Gottschlich, and G.~Katz.
\newblock {An Abstraction-Based Framework for Neural Network Verification}.
\newblock In {\em Proc. 32nd Int. Conf. on Computer Aided Verification (CAV)},
  pages 43--65, 2020.

\bibitem{GeMiDrTsChVe18}
T.~Gehr, M.~Mirman, D.~Drachsler-Cohen, E.~Tsankov, S.~Chaudhuri, and
  M.~Vechev.
\newblock {AI2: Safety and Robustness Certification of Neural Networks with
  Abstract Interpretation}.
\newblock In {\em Proc. 39th IEEE Symposium on Security and Privacy (S\&P)},
  2018.

\bibitem{GoodBengCour16}
I.~Goodfellow, Y.~Bengio, and A.~Courville.
\newblock {\em {Deep Learning}}.
\newblock MIT Press, 2016.

\bibitem{heZaReSu2015deep}
K.~He, X.~Zhang, S.~Ren, and J.~Sun.
\newblock {Deep Residual Learning for Image Recognition}.
\newblock In {\em Proc. IEEE Conf. on Computer Vision and Pattern Recognition
  (CVPR)}, pages 770--778, 2016.

\bibitem{HuKwWaWu17}
X.~Huang, M.~Kwiatkowska, S.~Wang, and M.~Wu.
\newblock {Safety Verification of Deep Neural Networks}.
\newblock In {\em Proc. 29th Int. Conf. on Computer Aided Verification (CAV)},
  pages 3--29, 2017.

\bibitem{JuLoBrOwKo2016}
K.~Julian, J.~Lopez, J.~Brush, M.~Owen, and M.~Kochenderfer.
\newblock {Policy Compression for Aircraft Collision Avoidance Systems}.
\newblock In {\em Proc. 35th Digital Avionics Systems Conf. (DASC)}, pages
  1--10, 2016.

\bibitem{Katz2017Reluplex}
G.~Katz, C.~Barrett, D.~Dill, K.~Julian, and M.~Kochenderfer.
\newblock {Reluplex: An Efficient SMT Solver for Verifying Deep Neural
  Networks}.
\newblock In {\em Proc. 29th Int. Conf. on Computer Aided Verification (CAV)},
  pages 97--117, 2017.

\bibitem{Katz2019Marabou}
G.~Katz, D.~Huang, D.~Ibeling, K.~Julian, C.~Lazarus, R.~Lim, P.~Shah,
  S.~Thakoor, H.~Wu, A.~Zelji\'c, D.~Dill, M.~Kochenderfer, and C.~Barrett.
\newblock {The Marabou Framework for Verification and Analysis of Deep Neural
  Networks}.
\newblock In {\em Proc. 31st Int. Conf. on Computer Aided Verification (CAV)},
  2019.

\bibitem{KiKiKiKiKi2019}
B.~Kim, H.~Kim, K.~Kim, S.~Kim, and J.~Kim.
\newblock {Learning Not to Learn: Training Deep Neural Networks With Biased
  Data}.
\newblock In {\em Proc. IEEE Conf. on Computer Vision and Pattern Recognition
  (CVPR)}, pages 9004--9012, 2019.

\bibitem{Liu2021AlgsVNN}
C.~Liu, T.~Arnon, C.~Lazarus, C.~Barrett, and M.~Kochenderfer.
\newblock {Algorithms for Verifying Deep Neural Networks}, 2020.
\newblock Technical Report. \url{http://arxiv.org/abs/1903.06758}.

\bibitem{MuMaSiPuVe22}
M.~M\"uller, G.~Makarchuk, G.~Singh, M.~P\"uschel, and M.~Vechev.
\newblock {PRIMA: General and Precise Neural Network Certification via Scalable
  Convex Hull Approximations}.
\newblock In {\em Proc. 49th ACM SIGPLAN Symposium on Principles of Programming
  Languages (POPL)}, 2022.

\bibitem{NaKaRySaWa17}
N.~Narodytska, S.~Kasiviswanathan, L.~Ryzhyk, M.~Sagiv, and T.~Walsh.
\newblock {Verifying Properties of Binarized Deep Neural Networks}, 2017.
\newblock Technical Report. \url{http://arxiv.org/abs/1709.06662}.

\bibitem{PrAf20}
P.~Prabhakar and Z.~Afzal.
\newblock {Abstraction based Output Range Analysis for Neural Networks}, 2020.
\newblock Technical Report. \url{http://arxiv.org/abs/2007.09527}.

\bibitem{SiGePuVe19}
G.~Singh, T.~Gehr, M.~Puschel, and M.~Vechev.
\newblock {An Abstract Domain for Certifying Neural Networks}.
\newblock In {\em Proc. 46th ACM SIGPLAN Symposium on Principles of Programming
  Languages (POPL)}, 2019.

\bibitem{SoKiPaShLe2020}
H.~Song, M.~Kim, D.~Park, Y.~Shin, and J.-G. Lee.
\newblock {End to End Learning for Self-Driving Cars}, 2020.
\newblock Technical Report. \url{http://arxiv.org/abs/2007.08199}.

\bibitem{TjXiTe17}
V.~Tjeng, K.~Xiao, and R.~Tedrake.
\newblock {Evaluating Robustness of Neural Networks with Mixed Integer
  Programming}, 2017.
\newblock Technical Report. \url{http://arxiv.org/abs/1711.07356}.

\bibitem{WaPeWhYaJa18}
S.~Wang, K.~Pei, J.~Whitehouse, J.~Yang, and S.~Jana.
\newblock {Formal Security Analysis of Neural Networks using Symbolic
  Intervals}.
\newblock In {\em Proc. 27th USENIX Security Symposium}, 2018.

\bibitem{WaZhXuLiJaHsKo21}
S.~Wang, H.~Zhang, K.~Xu, X.~Lin, S.~Jana, C.-J. Hsieh, and Z.~Kolter.
\newblock {Beta-CROWN: Efficient Bound Propagation with Per-Neuron Split
  Constraints for Complete and Incomplete Neural Network Verification}.
\newblock In {\em Proc. 35th Conf. on Neural Information Processing Systems
  (NeurIPS)}, 2021.

\bibitem{WuOzZeIrJuGoFoKaPaBa20}
H.~Wu, A.~Ozdemir, A.~Zelji\'c, A.~Irfan, K.~Julian, D.~Gopinath, S.~Fouladi,
  G.~Katz, C.~P\u{a}s\u{a}reanu, and C.~Barrett.
\newblock {Parallelization Techniques for Verifying Neural Networks}.
\newblock In {\em Proc. 20th Int. Conf. on Formal Methods in Computer-Aided
  Design (FMCAD)}, pages 128--137, 2020.

\bibitem{WuZeKaBa22}
H.~Wu, A.~Zelji\'c, K.~Katz, and C.~Barrett.
\newblock {Efficient Neural Network Analysis with Sum-of-Infeasibilities}.
\newblock In {\em Proc. 28th Int. Conf. on Tools and Algorithms for the
  Construction and Analysis of Systems (TACAS)}, pages 143--163, 2022.

\bibitem{JPhCS1168b2022Y2019}
X.~Ying.
\newblock {An Overview of Overfitting and its Solutions}.
\newblock {\em Journal of Physics Conference Series}, 1168:022022, 2019.

\end{thebibliography}

 }

 \newpage
 \setcounter{section}{0}
\def\thesection{\Alph{section}}

\noindent
{\huge{Appendix}}

\section{High-Level Overview of the CEGAR Process}\label{appendix:CEGAR}
	
	\begin{figure}
		\centering
		\begin{tikzpicture}[node distance=2cm]
			\node (origsys) [io] {$ N $};
			\node (abstract) [process, above of=origsys, yshift=0cm] {abstract};
			\node (abssys) [io, right of=abstract, xshift=0.6cm] {$ N' $};
			\node (verify) [process, right of=abssys, xshift=0.6cm] {verify};
			\node (isunsat) [decision, below right of=verify, yshift=-0.5cm, xshift=1.2cm] {is unsat};
			\node (origunsat) [startstop, left of=isunsat, xshift=-0.5cm] {unsat};
			\node (countex) [io, below of=isunsat, yshift=-0.3cm] {ce};
			\node (verifyex) [process, left of=countex, xshift=-0.6cm] {check at $ N $};
			\node (istrueex) [decision, left of=verifyex, xshift=-0.6cm] {is sat};
			\node (sat) [startstop, left of=istrueex, xshift=-0.6cm] {sat};
			\node (refine) [process, above of=istrueex, yshift=0.3cm] {refine};
			\draw [arrow] (origsys) -- (abstract);
			\draw [arrow] (abstract) -- (abssys);
			\draw [arrow] (abssys) -- (verify);
			\draw [arrow] (verify) -| (isunsat);
			\draw [arrow] (isunsat) -- node[above] {yes} (origunsat);
			\draw [arrow] (isunsat) -- node[right] {no} (countex);
			\draw [arrow] (countex) -- (verifyex);
			\draw [arrow] (verifyex) -- (istrueex);
			\draw [arrow] (istrueex) -- node[above] {yes} (sat);
			\draw [arrow] (istrueex) -- node[right] {no} (refine);
			\draw [arrow] (refine) -- (abssys);
		\end{tikzpicture}
		\captionof{figure}{DNN verification with abstraction-refinement.}
		\label{fig:AR-mechanism}
	\end{figure}
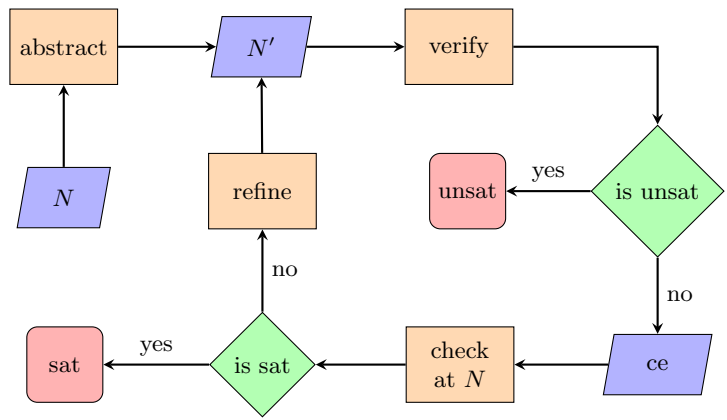

	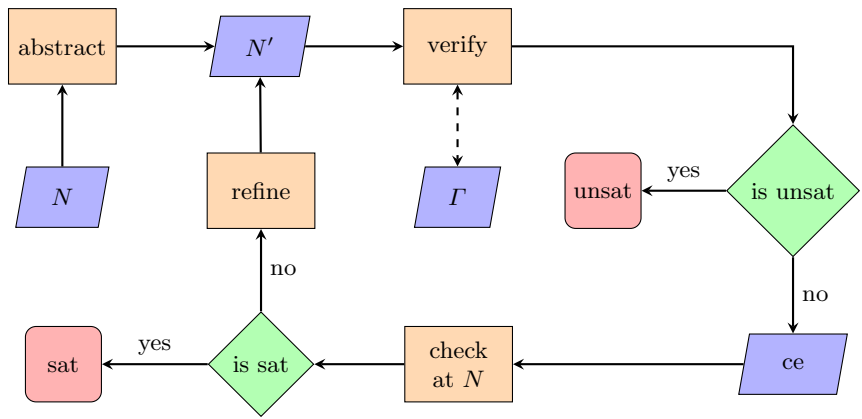
\begin{figure}
		\centering
		\begin{tikzpicture}[node distance=2cm]
			\node (origsys) [io] {$ N $};
			\node (abstract) [process, above of=origsys, yshift=0cm] {abstract};
			\node (abssys) [io, right of=abstract, xshift=0.6cm] {$ N' $};
			\node (verify) [process, right of=abssys, xshift=0.6cm] {verify};
			\node (gamma) [io, below of=verify, yshift=0cm] {$ \Gamma $};
			\node (isunsat) [decision, below right of=verify, yshift=-0.5cm, xshift=3cm] {is unsat};
			\node (origunsat) [startstop, left of=isunsat, xshift=-0.5cm] {unsat};
			\node (countex) [io, below of=isunsat, yshift=-0.3cm] {ce};
			\node (verifyex) [process, left of=countex, xshift=-2.4cm] {check at $ N $};
			\node (istrueex) [decision, left of=verifyex, xshift=-0.6cm] {is sat};
			\node (sat) [startstop, left of=istrueex, xshift=-0.6cm] {sat};
			\node (refine) [process, above of=istrueex, yshift=0.3cm] {refine};
			\draw [arrow] (origsys) -- (abstract);
			\draw [arrow] (abstract) -- (abssys);
			\draw [arrow] (abssys) -- (verify);
			\draw [arrow] (verify) -| (isunsat);
			\draw [arrow] (isunsat) -- node[above] {yes} (origunsat);
			\draw [arrow] (isunsat) -- node[right] {no} (countex);
			\draw [arrow] (countex) -- (verifyex);
			\draw [arrow] (verifyex) -- (istrueex);
			\draw [arrow] (istrueex) -- node[above] {yes} (sat);
			\draw [arrow] (istrueex) -- node[right] {no} (refine);
			\draw [arrow] (refine) -- (abssys);
			\draw [arrow, dashed, <->] (verify) -- (gamma);
		\end{tikzpicture}
		\captionof{figure}{DNN verification with residual reasoning.}
		\label{fig:AR-RR-mechanisms}
	\end{figure}
	
	\section{RR-Based Pruning Theorem}\label{appendix:formalization}
	\subsection{Lemma \ref{abstract min positive->refinement positive}: Proof}
	
	We add here the proof for the symmetric lemma regarding \textit{dec} Neurons.
	\begin{proof}\label{proof:abstract min positive->refinement positive}
		assume the two right nodes in Figure \ref{fig:refinednetwork} are the refined nodes of the right node in the right of Figure \ref{fig:abstractnetwork}. we have to prove that the following implication holds:
		$ x_1 \cdot min(a,b) + x_2 \cdot min(c,d) > 0 \Rightarrow (x_1 \cdot a + x_2 \cdot c > 0 \lor x_1 \cdot b + x_2 \cdot d > 0) $
		
		actually the values of $ x_1,x_2 $ are non-negative so we can split into 4 cases:
		\begin{enumerate}
			\item $ x_1=0, x_2=0 $ implication trivially holds
			\item $ x_1=0, x_2>0 $ in this case $ x_2 \cdot min(c,d) >0 $ so $ c,d>0 $ so both $ x_1 \cdot a+x_2 \cdot c=x_2 \cdot c>0 $ and $ x_1 \cdot b+x_2 \cdot d=x_2 \cdot d>0 $ and implication holds.
			\item $ x_1>0, x_2=0 $ in this case $ x_1 \cdot min(a,b)>0 $ so $ a,b>0 $ so both $ x_1 \cdot a+x_2 \cdot c=x_1 \cdot a>0 $ and $ x_1 \cdot b+x_2 \cdot d=x_1 \cdot b>0 $ and implication holds.
			\item If $ x_1>0, x_2>0 $ the implication is $ min(x_1 \cdot a,x_1 \cdot b) + min(x_2 \cdot c,x_2 \cdot d) > 0 \Rightarrow (x_1 \cdot
			a + x_2 \cdot c > 0 \lor x_1 \cdot b + x_2 \cdot d > 0) $.
			We can denote $ a'=x_1 \cdot a,\ b'=x_1\cdot b,\ c'=x_2\cdot c,\ d'=x_2\cdot d $ and we get the following lemma: 
			$ min(a',b')+min(c',d') > 0 \Rightarrow a'+c'>0 \lor b'+d'>0 $ (which is exactly the original lemma when $ x_1=1,x_2=1 $).
			\begin{itemize}
				\item If $ a' \le b' $, so $ a'=min(a',b') $ so by assumption $ a'+min(c',d') > 0 $, so $ b'+d' \ge a' + d' \ge a'+min(c',d') > 0 $
				\item If $ a' > b' $, so $ b'=min(a',b') $ so by assumption $ b'+min(c',d') > 0 $, so $ a'+c' > b' + c' > b' + min(c',d') > 0 $
			\end{itemize}
			So we get that there is a violation in at least one of the nodes in the refined network, and the implication does hold.
		\end{enumerate}
	\end{proof}
	
	\subsection{RR-Based Pruning Theorem}
	The detailed theorem and proof regarding $ \Gamma $ maintenance and derivations are supplied here.
		
	Assuming that $ \Gamma $ was maintained correctly in previous steps of the process of verification, one cannot derive unsatisfiability from $ \Gamma $ after refinement take place, since the abstract and refined networks are (similar but) not equivalent.
		
	The question is, therefore: What is needed to guarantee correct derivation (after refinement) from one clause in $ \Gamma $ to it's \textit{refined clause}? The \textit{refined clause} is the clause where the abstract node is replaced by its refined neurons in every case split in the clause. For example, in the running example, a case split $ v_{2,1+2}=active $ is replaced by two case splits: $ v_{2,1}=active,\ v_{2,2}=active $.
	
	We prove that the \textit{guards} whose conjunction guarantees the correctness of the implication are three: (i) equivalence (between abstract and refined networks) of all activations in all neurons in the preceding layers (including the abstract layer neurons, except the abstract neuron), (ii) equivalence (between abstract and refined networks) of activation to the abstract and refined neurons, and (iii) active/inactive case split for every $ \inc$/$\dec $ neuron (resp.) in the consecutive layers in both abstract and refined network.
	
	The guard condition $\guard{}$ is based on $\Gamma$'s clauses,
        as well as information gathered during the refinement process;
        specifically, the mapping between \relu{} nodes of the
        abstract and refined networks, and information regarding the
        earliest layer where an abstraction step has been
        performed. The guard condition is a conjunction of two parts:
        the first part is extracted from $ \Gamma $, and ensures that for every
        clause, for an \inc{} neuron, the negation of the variable of
        the abstract neuron is replaced with the disjunction of the
        two negated variables of its refined neurons. Similarly, for a \dec{}
        neuron, the variable of the abstract neuron is replaced with
        the disjunction of the two variables of its refined
        neurons. The second part of the guard condition is a
        conjunction of constraints that guarantee that each 
        \inc{} neuron is active, and that each \dec{} neuron is
        inactive --- if these neurons are located in a layer that
        appears after the earliest layer in which abstraction
        was performed. 
	
	In order to formalize the answer, some sets of constraints are to be defined. Suppose we are given a property $ \phi $ of the form $ c_1\le x \le c_2 \rightarrow y > c_3 $ for some constants $ c_1,c_2,c_3 $, and a triple of original, abstract and refined networks $ N, N', N'' $ resp, such that $ N'' $ refined from $ N' $ with 1 refinement step on node $ u $ in layer $ L_r $. Denote:
	\begin{enumerate}
		\item $ C_{pre}:= \bigwedge\limits_{p\in P} \{p=active/inactive\} $ where $ P \subseteq \{L_1\cup\dots \cup L_r\setminus u\} $. 
		
		$ C_{pre} $ include case splits of neurons in preceding layers (including the abstract layer, except the abstract neuron). Each neuron $ p\in P $ can be active/inactive.
		\item $ C^{inc}_{abs}:= \{u=active\} $ if $ u $ is an increasing node else $ \emptyset $
		\item $ C^{dec}_{abs}:= \{u=inactive\} $ if $ u $ is a decreasing node else $ \emptyset $
		\item $ C^{inc}_{ref}:= \{u_1=active, u_2=active\} $ if $ u $ is an increasing node else $ \emptyset $
		\item $ C^{dec}_{ref}:= \{u_1=inactive, u_2=inactive\} $ if $ u $ is a decreasing node else $ \emptyset $
		\item $ C^{inc}_{post}:= \bigwedge\limits_{q\in Q^{inc}} \{q = active\} $ where $ Q^{inc} \subseteq \{L_r\cup\dots \cup L_|N|\} $ is a set of all increasing nodes in layers $ L_{r+1},\dots, L_{|N|} $.
		
		$ C^{inc}_{post}$ requires that every $ \inc $ neuron in the consecutive layers after the abstract node will be \textit{active}.
		\item $ C^{dec}_{post}:= \bigwedge\limits_{q\in Q^{dec}} \{q = inactive\} $ where $ Q^{dec} \subseteq \{L_r\cup\dots \cup L_|N|\} $ is a set of all decreasing nodes in layers $ L_{r+1},\dots, L_{|N|} $.
		
		$ C^{dec}_{post}$ requires that every $ \dec $ neuron in the consecutive layers after the abstract node will be \textit{inactive}.
	\end{enumerate}
	Using these definitions, the 3 guards above are defined as follows:
	\begin{itemize}
		\item $ N'_{C_{pre}}=N''_{C_{pre}} $ implements guard (i).
		\item$ N'_{C^{inc}_{abs}} + N'_{C^{dec}_{abs}} = N''_{C^{inc}_{ref}} + N''_{C^{dec}_{ref}} $ implements guard (ii).
		\item $ (N'_{C^{inc}_{post}}+N'_{C^{dec}_{post}}) \land (N''_{C^{inc}_{post}}+N''_{C^{dec}_{post}}) $ implements guard (iii).
	\end{itemize}
	
	We should prove, given that all guards holds, the correctness of implication of unsatisfiability in refinement from unsatisfiability in abstraction. We prove that by finding the violated case split in the refinement.
	
	Notice that a possible solution from the verifier is an input $ \vec{x} $ that (when propagated in the abstract network) induces and a set of case splits $ \{s_i\}_{i=1}^{n} $, hence the solution can be treated as a couple $ (\vec{x}, \{s_i\}_{i=1}^{n}) $.
		
	\begin{theorem}\label{appendix-theorem:prune}
		If all 3 guards hold, then for every solution $ (\vec{x}, \{s_i\}_{i=1}^{n}) $, if any constraint $ s\in \{s_i\}_{i=1}^{n} $ is violated in $ N' $, there is a corresponding constraint which is violated in $ N'' $.
	\end{theorem}
	\begin{proof}
		The violation can occur in any neuron in the abstract network, and we handle all the options by splitting the neurons to input layer, preceding layers, abstract neuron, consecutive layers, output layer.
		\begin{enumerate}
			\item{If $ s $ is an input constraint (denoted as $ C_{IN} $), $ s $ is also violated at $ N'' $}.
			\item{If $ s $ is an output constraint (denoted as $ C_{OUT} $), $ s $ is also violated at $ N'' $ because the output of the refinement is smaller (and $ y''<y'<c $)}.
			\item {If $ s \in N'_{C_{pre}} $, then from guard (i) we get that $ s \in N''_{C_{pre}}$, hence $ s $ is violated in $ N'' $ because the values in the preceding layers are equal.}
			\item {If $ s\in N'_{C^{inc}_{abs}}+N'_{C^{dec}_{abs}} $ then from guard (ii) we get that $ s \in N''_{C^{inc}_{ref}}+N''_{C^{dec}_{ref}} $.
				\subitem {if $ u $ is \textit{Inc} and \textit{active} then violation means $ u<0 $, which implies that $ u_1<0 \lor u_2<0 $ (from Lemma \ref{abstract max negative->refinement negative}) so some $ s_1\in N''_{C^{inc}_{ref}} $ is violated in $ N'' $ }.
				\subitem {if $ u $ is \textit{Dec} and \textit{inactive} then violation means $ u>0 $, which implies that $ u_1>0 \lor u_2>0 $ (from Lemma \ref{abstract min positive->refinement positive}) so some $ s_2 \in N''_{C^{dec}_{ref}} $ is violated in $ N'' $.}}
			\item {if $ s\in N'_{C^{inc}_{post}}+N'_{C^{dec}_{post}} $ then from guard (iii) we get that $ s\in N''_{C^{inc}_{post}}+N''_{C^{dec}_{post}} $} (since guard (iii) also implies that $ N'_{C^{inc}_{post}}+N_{C^{dec}_{post}} = N''_{C^{inc}_{post}}+N''_{C^{dec}_{post}} $)
				\subitem {if any constraint in $ C_{IN}, C_{OUT},C_{pre},C_{abs} $ is violated, we already shown that a corresponding constraint is violated}
				\subitem {otherwise, denote the violated constraint's neuron with $ p $}
				\subsubitem {if $ p $ is an $ \inc $ neuron, then the violation is that $ p<0 $. After refinement $ p $ decreases so still $ p<0 $ and $ s $ is violated again.}
				\subsubitem {if $ p $ is a $ \dec $ neuron, then the violation is that $ p>0 $. After refinement $ p $ increases so still $ p>0 $ and $ s $ is violated again.}
		\end{enumerate}
	\end{proof}
	
	\section{\arf{} Side Procedures: Additional Details}\label{appendix:ar4-side-procedures}
	The calculus in Figure \ref{fig:ar4-calculus} uses some procedures whose pseudo-code is given at this part. Algorithm \ref{alg:can_abstract} checks if 2 neurons can be abstracted into one by their type, Algorithm \ref{alg:abstract step} update Reluplex variables, basis, valuation function and bounds, unify 2 neurons and update the tableau. Algorithm \ref{alg:abstractSequence} just apply sequence of AbstractStep procedures. Algorithm \ref{alg:UnifyBasic} unifies two variables which are known to be part of the basis, by replacing their 2 rows by 1 row of the unified neuron variable in the tableau (a data structure used in Reluplex, and maintains the data regarding variables values, bounds, etc.). Algorithm \ref{alg:isRealSAT} validate counter example in the original network.
	
	\begin{algorithm}[ht!]
		\caption{CanAbstract($ v_1, v_2 $)}
		\begin{algorithmic}[1]
			\label{alg:can_abstract}
			\IF {$ layer(v_1) = layer(v_2) $ \\
				\hskip\algorithmicindent\AND $ pos(v_1) \leftrightarrow pos(v_2) $ \\
				\hskip\algorithmicindent\AND $ inc(v_1) \leftrightarrow inc(v_2) $ \\
				\hskip\algorithmicindent\AND $\{v_1^b, v_1^f, v_2^b, v_2^f\} \subseteq B $ \\
				\hskip\algorithmicindent\AND $\forall v\in\{v_1^b,v_1^f,v_2^b,v_2^f\}: l(v)= -\infty $ \\
				\hskip\algorithmicindent\AND $\forall v\in\{v_1^b,v_1^f,v_2^b,v_2^f\}: u(v)=\infty $}
			\RETURN True
			\ENDIF 
			\RETURN False
		\end{algorithmic}
		\label{alg:canAbstract}
	\end{algorithm}
	
	\begin{algorithm}[ht!]
		\caption{AbstractStep($  X,T,\alpha, l, u, B, v_1, v_2 $)}
		\begin{algorithmic}[1]
			\STATE $ X:=X\cup\{v_{1,2}^b,v_{1,2}^f\}\setminus\{v_1^b,v_1^f,v_2^b,v_2^f\} $
			\STATE $ B := B\cup \{v_{1,2}^b,v_{1,2}^f\} \setminus \{v_1^b,v_1^f,v_2^b,v_2^f\}$
			\STATE $ \alpha := \alpha \cup \{v_{1,2}^b=0,v_{1,2}^f=0\}\setminus \{v_1^b=0,v_1^f=0, v_2^b=0, v_2^f=0\}$
			\STATE $ l:=l \cup \{v_{1,2} > -\infty\} \setminus \{v_1>-\infty,v_2>-\infty\} $
			\STATE $ u:=u \cup \{v_{1,2} < \infty\} \setminus \{v_1<\infty,v_2<\infty\} $
			\STATE $ T:=UnifyBasic(v_{1,2}, (v_1,v_2), T) $
			\RETURN $ X,T,\alpha, l, u, B $
		\end{algorithmic}
		\label{alg:abstract step}
	\end{algorithm}
	
	\begin{algorithm}[ht!]
		\caption{Abstract($ X_0,T_0,\alpha_0, l_0, u_0, B_0,\Gamma_A $)}
		\begin{algorithmic}[1]
			\STATE $ X,T,\alpha, l, u, B = X_0,T_0,\alpha_0, l_0, u_0, B_0 $
			\FOR{$ v_{1,2}, v_1, v_2 \in \Gamma_A $}
			\STATE $ X,T,\alpha, l, u, B $ = abstractStep($ X,T,\alpha, l, u, B, v_1, v_2 $)
			\ENDFOR
			\RETURN $ X,T,\alpha, l, u, B $
		\end{algorithmic}
		\label{alg:abstractSequence}
	\end{algorithm}
	
	\begin{algorithm}[ht!]
		\caption{UnifyBasic($ v_{1,2}, (v_1, v_2), T $)}
		\begin{algorithmic}[1]
			\STATE remove the rows of $ v_1,v_2 $
			\STATE newRowAddAnds = $ [0,\dots,0] $
			\FOR{$ i\in \{0, \dots, len(X)\} $}
			\IF {$ x_i == v_{1,2}$}
			\STATE $ newRowAddAnds[i] = 1 $
			\ELSIF {$ i \in [n] $} 
			\IF {$ inc(x_i) $}
			\STATE $ newRowAddAnds[i] = max(x_i) $
			\ELSE
			\STATE $ newRowAddAnds[i] = min(x_i) $
			\ENDIF
			\ELSE
			\STATE $ newRowAddAnds[i] = 0 $
			\ENDIF
			\ENDFOR
			\STATE append newRowAddAnds to $ T $
		\end{algorithmic}
		\label{alg:UnifyBasic}
	\end{algorithm}
	
	\begin{algorithm}[ht!]
		\caption{IsRealSAT($ \alpha $)}
		\begin{algorithmic}[1]
			\STATE derive $ \alpha_0 $ from $ \alpha $
			\STATE - take only input layer values from $ \alpha $
			\STATE - evaluate in the original network and derive other values of $ X_0 $
			\RETURN {$ \alpha_0 $ satisfies $ T_0 $}
		\end{algorithmic}
		\label{alg:isRealSAT}
	\end{algorithm}
	
	\vfill
	
	
	
\end{document}